\documentclass{article}

\usepackage{amssymb}
\usepackage{amsmath}
\usepackage{amsthm}
\usepackage{bbm}
\usepackage{bm}
\usepackage{breakcites}
\usepackage{centernot}
\usepackage{comment}
\usepackage{dsfont}
\usepackage[inline, shortlabels]{enumitem}
\usepackage{graphicx}
\usepackage{mathtools}
\usepackage{microtype}
\usepackage{parskip}
\usepackage{scalerel}
\usepackage{setspace}
\usepackage[normalem]{ulem}

\usepackage{hyperref}
\usepackage[capitalise]{cleveref}

\usepackage[accepted]{icml2021}

\newcommand{\thetitle}{Provably Strict Generalisation Benefit for Equivariant Models}
\icmltitlerunning{\thetitle}

\theoremstyle{definition}
\newtheorem{theorem}{Theorem}

\newtheorem{lemma}[theorem]{Lemma}

\newtheorem{example}[theorem]{Example}

\newtheorem{proposition}[theorem]{Proposition}
\newtheorem{remark}[theorem]{Remark}

\newtheorem*{theorem*}{Theorem}
\newtheorem*{proposition*}{Proposition}
\newtheorem*{definition*}{Definition}

\newtheorem*{question*}{Question}

\newtheorem*{claim*}{Claim}

\newcommand\argmin{\operatorname*{argmin}}

\DeclarePairedDelimiterX{\twobarparen}[2]{(}{)}{#1\;\delimsize\|\;#2}
\DeclarePairedDelimiterX{\inner}[2]{\langle}{\rangle}{#1, #2}
\DeclarePairedDelimiterX{\binner}[2]{\lparen}{\rparen}{#1 \vert{} #2}

\let\grad\nabla

\newcommand{\dd}{\mathop{}\!\mathrm{d}}
\newcommand{\abs}[1]{\left\vert{} #1 \right\vert{}}

\newcommand{\G}{\mathcal{G}}

\newcommand{\T}{\mathcal{T}}

\newcommand{\F}{\mathcal{F}}

\DeclareMathOperator{\E}{\mathbb{E}}
\newcommand{\var}[1]{\text{Var}\left[#1 \right]}

\newcommand{\indep}{\perp\!\!\!\perp}

\DeclareMathOperator{\Unif}{\text{Unif}}
\newcommand{\n}{\mathcal{N}}
\newcommand{\iid}{i.i.d.}

\newcommand{\SO}{\text{SO}}
\newcommand{\GL}{\text{GL}}

\renewcommand{\H}{\mathcal{H}}
\newcommand{\N}{\mathbb{N}}
\newcommand{\R}{\mathbb{R}}
\newcommand{\Y}{\mathcal{Y}}
\newcommand{\X}{\mathcal{X}}

\newcommand{\norm}[1]{\lVert#1\rVert}
\newcommand{\lnorm}[1]{\left\lVert#1\right\rVert}
\newcommand{\fnorm}[1]{\norm{#1}_\text{F}}
\newcommand{\vc}[1]{\text{VC}(#1)}
\DeclareMathOperator{\Rad}{\widehat{\mathfrak{R}}}
\DeclareMathOperator{\ERad}{\mathfrak{R}}

\DeclareMathOperator{\tr}{Tr} 

\newcommand{\ee}{\mathsf{e}}

\newcommand{\eqas}{\overset{\text{a.s.}}{=}}
\newcommand{\eqdist}{\overset{\text{d}}{=}}

\DeclareUnicodeCharacter{2212}{$-$}

\newcommand{\sperp}{{\scaleobj{0.75}{\perp}}}

\renewcommand{\cite}{\citep}
\newcommand{\Esig}{\E_{\bm{\sigma}}}
\newcommand{\x}{{\bm{X}}}
\newcommand{\y}{\bm{Y}}
\renewcommand{\O}{\mathcal{O}}
\newcommand{\q}{\mathcal{Q}}
\newcommand{\vlin}{{V_\text{lin}}}
\newcommand{\wlin}{{W_\text{lin}}}
\newcommand{\Gfin}{\G_{\text{fin}}}
\newcommand{\Hom}{\text{Hom}}
\newcommand{\Gr}{\mathbb{G}}
\newcommand{\twine}{{\Psi_\G}}
\newcommand{\proj}{{\Phi_\G}}

\def\submission{0}
\newenvironment{supplementary}{}{}

\newcommand{\titlestuff}{
\icmltitle{\thetitle}

\icmlsetsymbol{equal}{*}

\begin{icmlauthorlist}
\icmlauthor{Bryn Elesedy}{cs}
\icmlauthor{Sheheryar Zaidi}{stats}
\end{icmlauthorlist}

\icmlaffiliation{cs}{Department of Computer Science, University of Oxford, Oxford, United Kingdom}
\icmlaffiliation{stats}{Department of Statistics, University of Oxford, Oxford, United Kingdom}

\icmlcorrespondingauthor{Bryn Elesedy}{bryn@robots.ox.ac.uk}

\icmlkeywords{Equivariance, Generalisation, Generalization, Machine Learning, ICML}

\vskip 0.3in
}

\begin{document}

\if\submission1
    \def\figwidth{0.45\textwidth}
    \excludecomment{proof}
    \excludecomment{supplementary}
    \twocolumn[\titlestuff]
\else
    \def\figwidth{0.85\textwidth}
    \onecolumn
    \titlestuff
\fi

\printAffiliationsAndNotice{}  

\begin{abstract}
    It is widely believed that engineering a model to be invariant/equivariant improves generalisation.
    Despite the growing popularity of this approach,
    a precise characterisation of the generalisation benefit is lacking.
    By considering the simplest case of linear models,
    this paper provides the first \emph{provably non-zero} improvement in generalisation
    for invariant/equivariant models 
    when the target distribution is
    invariant/equivariant
    with respect to a compact group.
    Moreover, our work reveals an interesting relationship between
    generalisation, the number of training examples and properties of the 
    group action.
    Our results rest on an observation of the structure of function spaces under
    averaging operators which, along with its consequences for feature
    averaging, may be of independent interest.
\end{abstract}

\section{Introduction}
There is significant and growing interest in models,
especially neural networks, that are invariant or equivariant to the action of a group 
on their inputs.
It is widely believed that these models enjoy improved generalisation when the group is correctly specified.
The intuition being that if the salient aspects of a task are unchanged by some transformation,
then more flexible models would need to learn (as opposed to being hard-coded) to
ignore these transformations, requiring more samples to generalise.
Work in this area has progressed quickly~\cite{cohen2016group,cohen2018spherical,cohen2019general}
and has found application in domains where the symmetry is known a priori, for instance
in particle physics~\cite{pfau2020ab}.

In contrast with practical successes, our theoretical understanding of invariant/equivariant models is limited.
Many previous works that have attempted to address the generalisation of invariant/equivariant models,
such as~\citet{sokolic2017generalization,sannai2019improved},
cover only the worst case performance of algorithms.
These works use complexity measures to find tighter upper bounds on the test
risk for invariant/equivariant models, but a strict benefit is not demonstrated.
The VC dimension governs distribution independent generalisation, 
so these complexity measures 
can show no more than a separation between the VC dimensions of a model
and its invariant/equivariant version.
This would imply the existence of a distribution on
which the model with smaller VC dimension will generalise better,
but would not rule out that on many common distributions
training an invariant/equivariant model using standard procedures 
provides no benefit.
For instance, there could be many invariant/equivariant distributions on which
SGD automatically favours parameters that result in (possibly approximately)
invariant/equivariant predictors, regardless of architecture.

The results of this paper move to address this issue, by quantifying exactly the generalisation
benefit of invariant/equivariant linear models.
We do this in the case of the minimum $L_2$-norm least-squares solution, which reflects the implicit
bias of gradient descent in overparameterised linear models.
While the linear model provides a tractable first-step towards understanding more complex models such 
as neural networks, the underlying ideas of this paper are equally applicable to non-linear predictors.
We emphasise this by providing new perspectives on feature averaging and 
suggestions for how to apply the ideas of this paper to find new methods for training
invariant/equivariant neural networks.

\subsection{Our Contributions}
The main result of this paper is~\Cref{thm:equivariant-regression}, which 
quantifies the generalisation benefit of equivariance in a linear model.
We define the \emph{generalisation gap} $\Delta(f, f')$
between predictors $f$ and $f'$ to be the difference in their test errors on a given task.
A positive generalisation gap $\Delta(f, f') > 0$ means that $f'$ has strictly smaller test error
than $f$.
\Cref{thm:equivariant-regression} concerns 
$\Delta(f, f')$ in the case that $f: \R^d \to \R^k$ is the minimum-norm least-squares predictor
and $f'$ is its equivariant version. 
Let a compact group $\G$ act via orthogonal representations $\phi$ and $\psi$ on
inputs $X \sim \n(0, I_d)$ and outputs $Y = h(X) + \xi \in \R^k$ respectively,
where $h: \R^d \to \R^k$ is an equivariant linear map.
Let $\binner{\chi_\psi}{\chi_\phi}= \int_\G  \tr(\psi(g))\tr(\phi(g))\dd \lambda(g)$ denote
the scalar product of the characters of the representations. 
The generalisation benefit of enforcing equivariance in a linear model is given by
\[
    \E[\Delta(f, f')] = \var{\xi} r(n, d)(dk - \binner{\chi_\psi}{\chi_\phi}) + \mathcal{E}_\G(n, d)
\]
where
\[
    r(n, d) = \begin{cases}
        {\frac{n}{d(d - n - 1)}} & n < d- 1\\
        (n - d - 1)^{-1} & n > d + 1\\
        \infty & \text{otherwise}
    \end{cases}
\]
and $\mathcal{E}_\G(n, d) \ge 0 $ is the generalisation gap of the corresponding noiseless problem, that
vanishes when $n \ge d$.
The divergence at the interpolation threshold $n \in [d-1, d+1]$ 
is consistent with the double descent literature~\cite{hastie2019surprises}. 

The quantity $dk - \binner{\chi_\psi}{\chi_\phi}$ 
represents the significance of the group symmetry to
the task. The dimension of the space of linear maps $\R^d \to \R^k$ is $dk$,
while $\binner{\chi_\psi}{\chi_\phi}$ is the dimension of the space of equivariant linear maps. 
We will see later that the quantity $dk - \binner{\chi_\psi}{\chi_\phi}$ 
represents the dimension of the space of linear maps that vanish when averaged over $\G$;
it is through the dimension of this space that the symmetry in the task controls the generalisation gap.
Although invariance is a special case of equivariance, we find it instructive to
discuss it separately. In~\Cref{thm:invariant-regression} we provide
a result that is analogous to~\Cref{thm:equivariant-regression} for invariant predictors,
along with a separate proof.

In order to arrive at~\Cref{thm:invariant-regression,thm:equivariant-regression} we
make use of general results about the structure of function spaces under averaging operators.
In~\Cref{sec:decomposition} we show how averaging operators can be used to decompose
function spaces into orthogonal subspaces of symmetric (invariant/equivariant) 
and anti-symmetric (vanish when averaged) functions.
In~\Cref{sec:fa} we use these insights to provide new perspectives on feature averaging.
Our main results are in~\cref{sec:lower-bounds}.
Finally, in~\Cref{sec:neural-networks} we apply our insights to derive principled 
methods for training invariant/equivariant neural networks and provide open questions
for future work.

\section{Related Work}
\paragraph{Implementations}
While there has been a recent surge in interest,
symmetry is not a new concept in machine learning.
Recent literature is dominated by neural networks, but other methods do exist:
e.g.~kernels~\cite{haasdonk05invariancein}, support
vector machines~\cite{scholkopf96incorporatinginvariances} or feature 
spaces such as polynomials~\cite{schulz1994constructing,schulz1992existence}.
The engineering of invariant neural networks dates back at least to~\citet{wood1996representation},
in which ideas from representation theory are applied to find weight tying schemes that result
in group invariant architectures; similar themes are present in~\citet{ravanbakhsh2017equivariance}.
Recent work follows in this vein,
borrowing ideas from fundamental physics to construct invariant/equivariant
convolutional architectures~\cite{cohen2016group,cohen2018spherical}.
Correspondingly, a sophisticated theory of invariant/equivariant
networks has arisen~\cite{kondor2018generalization,cohen2019general}
including universal approximation results~\cite{maron2019universality,yarotsky2018universal}.

\paragraph{Learning and Generalisation}
The intuition that invariant or equivariant models are more sample efficient
or generalisable is widespread in the literature, but arguments are often heuristic and,
to the best of our knowledge, a provably strict (non-zero) generalisation benefit has not appeared before this paper.
It was noted~\cite{abu1993hints} that constraining a model to be invariant cannot increase its
VC dimension.
An intuitive argument for reduced sample complexity is made in~\citet{mroueh2015learning}
in the case that the input space has finite cardinality.
The sample complexity of linear classifiers with invariant representations
trained on a simplified image task is discussed briefly in~\citet{anselmi2014unsupervised},
the authors conjecture that a general result may be obtained using wavelet transforms.
The framework of robustness~\cite{xu2012robustness} is used in~\citet{sokolic2017generalization}
to obtain a generalisation bound for interpolating large-margin classifiers that are
invariant to a finite set of transformations;
note that the results contain an implicit margin constraint on the training data.
The generalisation of models invariant or equivariant
to finite permutation groups is considered in~\citet{sannai2019improved}.
Both of~\citet{lyle2019analysis,lyle2020benefits} cover the PAC Bayes approach to generalisation of invariant models,
the latter also considers the relative benefits of feature averaging and data augmentation.

\section{Preliminaries}
We assume familiarity with the basic notions of group theory, 
as well as the definition of a group action and a linear representation.
The reader may consult~\citet[Chapters 1-4]{wadsleyrep2012,serre1977linear} for background.
We define some key concepts and notation here and introduce more as necessary throughout the paper.

\paragraph{Notation and Technicalities}
We write $\X$ and $\Y$ for input and output spaces respectively.
We assume for simplicity that $\Y =(\R^k, +)$ is a $k$-dimensional
real vector space (with $k$ finite) but
we expect our results to apply in other settings too.
Throughout the paper, $\G$ will represent an
arbitrary compact group that has a measurable action $\phi$ on $\X$ and a representation $\psi$ on $\Y$.
By this we mean that $\phi: \G \times \X \to \X$ is a measurable map and the same for $\psi$.
We sometimes write $gx$ as a shorthand for $\phi(g) x$ and similarly for actions on $\Y$.
Some notation for specific groups: $C_m$ and $S_m$ are, respectively, the cyclic and symmetric groups on $m$ elements;
while $O(m)$ and $SO(m)$ are the $m$-dimensional orthogonal and special orthogonal groups respectively.
For any matrix $A$ we write 
$A^+$ for the Moore-Penrose pseudo-inverse
and $\fnorm{A} = \sqrt{\tr(A^\top A)}$ for the Frobenius/Hilbert-Schmidt norm.
We write $\Gr_n(\R^d)$ for the Grassmannian manifold of subspaces of dimension $n$ in $\R^d$.
The results of this paper require some mild care with technical considerations
such as topology/measurability.
We do not stress these in the main paper but they do appear
in the proofs, %
\if\submission1
all of which are deferred to the supplementary material. 
\else
some of which are deferred to the appendix.
\fi

\paragraph{Invariance, Equivariance and Symmetry}
A function $f: \X \to \Y$ is $\G$-invariant if
$ f(\phi(g)x) = f(x)$ $\forall x\in \X$ $\forall g\in \G$
and is $\G$-equivariant if
$ f(\phi(g)x) = \psi(g)f(x)$ $\forall x\in \X$ $\forall g\in \G$.
A measure $\mu$ on $\X$ is $\G$-invariant 
if $\forall g \in \G$ and any $\mu$-measurable $B \subset \X$ 
the pushforward of $\mu$ by the action $\phi$ equals $\mu$, i.e.~$(\phi(g)_* \mu)(B) = \mu(B)$.
This means that if $X \sim \mu$ then $\phi(g) X \sim \mu$ $\forall g \in \G$.
We will sometimes use the catch-all term symmetric to describe an object that is invariant
or equivariant.

\section{Symmetric and Anti-Symmetric Functions}\label{sec:decomposition}
Averaging the inputs of a function over the orbit of a group
is a well known method to enforce invariance, for instance see~\citet{schulz1994constructing}.
Approaching this from another perspective, averaging can also be used to identify invariance.
That is, a function is $\G$-invariant if and only if it is preserved by orbit averaging with
respect to $\G$. The same can be said for equivariant functions, using a modified average.
After introducing the relevant concepts, we will
use this observation and other properties of the averaging operators 
to decompose function spaces into mutually orthogonal
symmetric (invariant/equivariant) and anti-symmetric (vanish when averaged) subspaces.
This observation provides the foundation for many results later in the paper.

\subsection{Setup}
\paragraph{Haar Measure}
Let $\G$ be a compact group. The Haar measure is the unique invariant
measure on $\G$ and we denote it by $\lambda$.
By invariance we mean that for any measurable subset $A \subset \G$ and for any $g\in \G$,
$\lambda(gA) = \lambda(Ag) = \lambda(A)$. 
We assume normalisation such that $\lambda(\G) = 1$, which is always possible when $\G$ is compact.
The (normalised) Haar measure can be interpreted as the uniform distribution on $\G$.
See~\citet{kallenberg2006foundations} for more details.

\paragraph{Orbit Averaging}
For any feature map $f:\X \to \Y$,
we can construct a $\G$-invariant feature map by averaging with respect to $\lambda$.
We represent this by the operator $\O$, where
\[
    (\O f)(x) = \int_\G  f(gx)\dd \lambda(g).
\]
Similarly, if $\psi$ is a representation of $\G$ on $\Y$, we can transform $f$
into an equivariant feature map by applying $\q$, where
\[
    (\q f)(x) = \int_\G  \psi(g^{-1}) f(gx)\dd \lambda(g).
\]
Notice that $\O$ is a special case of $\q$ corresponding to 
$\psi$ being the trivial representation.
The operator $\O$ can be thought of as performing feature averaging with respect to $\G$.
This interpretation is widely adopted, for instance appearing in~\citet{lyle2020benefits}.

\paragraph{Function Spaces}
We now show how to construct the relevant spaces of functions.
We present this in an abstract way, but these functions can be interpreted 
as predictors, feature maps, feature extractors and so on.
Let $\mu$ be a $\G$-invariant measure on $\X$ and let $\inner{a}{b}$ for $a, b \in \Y $ be an inner product
on $\Y= \R^k$ that is preserved by $\psi$.%
\footnote{
    It is permissible for inner product itself to depend on the point $x$
    at which the feature maps are evaluated.
    The only requirement is that evaluating the inner product between
    two fixed vectors is a $\G$-invariant function $\inner{a}{b}(x) = \inner{a}{b}(gx)$,
    $\forall a, b \in \R^k$, $g \in \G$ and $x \in \X$.
    We believe that this allows our results to extend to the
    case of the features defined as maps from a manifold to its tangent
    bundle.}
By preserved we mean that
$\inner{\psi(g)a}{\psi(g)b} = \inner{a}{b}$, $\forall g\in \G$, $\forall a, b\in \Y$
and any inner product can be transformed to satisfy this property using the Weyl trick
$
    \inner{a}{b} \mapsto \int_\G\inner{ \psi(g)a }{ \psi(g)b }\dd\lambda(g)
    $.
Given two functions $f, h: \X \to \Y$, we define their inner product by
\[
    \inner{f}{h}_\mu = \int_\X \inner{f(x)}{h(x)} \dd \mu (x).
\]
This inner product can be thought of as comparing the similarity between functions
and can used to define a notion of distance with the norm $\norm{f}_\mu = \sqrt{\inner{f}{f}}_\mu$.
We then define $V$ as the space of all (measurable) functions 
$f: \X \to \Y$ such that $\norm{f}_\mu < \infty$.\footnote{Equality is defined $\mu$-almost-everywhere.}
Formally, $V$ is a Bochner space.

\subsection{Averaging and the Structure of Function Spaces}
We have seen how to define orbit averaging operators to produce invariant and equivariant functions as
well as how to construct spaces of functions on which these operators can act.
The reason for all of this is the following result, which shows that the averaging operators
allow us to decompose any function in $V$ into orthogonal $\G$-symmetric and $\G$-anti-symmetric parts.
Recall that since $\O$ is just a special case of $\q$,~\Cref{thm:fa} applies to both operators.

\begin{lemma}\label{thm:fa}
    Let $U$ be any subspace of $V$ that is closed under $\q$.
    Define the subspaces $S$ and $A$ of, respectively, the $\G$-symmetric and $\G$-anti-symmetric functions in
    $U$:
    $
        S = \{f \in U: \text{$f$ is $\G$-equivariant}\} 
        $
        and
        $
        A = \{f \in U: \q f = 0 \}.
        $
    Then $U$ admits admits an orthogonal decomposition into symmetric and anti-symmetric parts
    \[
        U = S \oplus A.
    \]
\end{lemma}
\begin{supplementary}
See~\Cref{sec:proofs-appdx} for the proof.
\end{supplementary}
The proof consists of establishing the following properties:
\begin{enumerate*}[(A)]
    \item for any $f \in V$, $\q f \in V$;
    \item any $f \in V$ is $\G$-equivariant if and only if $\q f =f$;
    \item $\q$ has only two eigenvalues, $1$ and $0$; and,
    \item $\q$ is self-adjoint with respect to $\inner{\cdot}{\cdot}_\mu$.
\end{enumerate*}
The last of these is critical and depends on the $\G$-invariance of $\mu$.
There are many tasks for which $\G$-invariance of the input distribution is a 
natural assumption, for instance in medical imaging~\cite{winkels20183d}.

\Cref{thm:fa} says that any function $u \in U$ can be written $u = s + a$,
where $s$ is $\G$-equivariant, $\q a  = 0$ and $\inner{s}{a}_\mu = 0$.
We refer to $s$ and $a$ as the symmetric and anti-symmetric parts of $u$.
In general this does not imply that $a$ is odd, that it 
outputs an anti-symmetric matrix or that it is negated by
swapping two inputs. 
These are, however, special cases.
If $\G = C_2$ acts by $x\mapsto -x$ then odd functions $f: \R
\to \R$ will be anti-symmetric in the sense of this paper. If $\G=C_2$ acts on matrices by $M \mapsto
M^\top$ then $f: M \mapsto \frac12 (M - M^\top)$ is also anti-symmetric. Finally, if
$\G = S_n$ and $f: \R^n \to \R$ with $f(x_1, \dots, x_j, x_{j+1}, \dots, x_n) =
-f(x_1, \dots, x_{j+1}, x_{j}, \dots, x_n)$, then $f$ is anti-symmetric in the sense of this paper.

Although it is straightforward to demonstrate
and has surely been observed before, we will see in the rest of the paper
that the perspective provided by~\Cref{thm:fa} is highly fruitful.
Before that, we conclude this section with an example for intuition.
\begin{example}\label{example:rotations}
    Let $V$ consist of all functions $f: \R^2 \to \R$ such that $\E[f(X)^2] < \infty$
    where $X \sim \n(0, I_2)$.
    Let $\G = \SO(2)$ act by rotation about the origin, with respect to which
    the normal distribution is invariant.
    Using~\Cref{thm:fa} we may write $V = S \oplus A$.
    Alternatively, consider polar coordinates $(r, \theta)$, then for any feature map $f(r, \theta)$
    we have $\O f (r, \theta) = \frac{1}{2\pi} \int_0^{2\pi} f(r, \theta)\dd \theta $. So any $\G$-invariant
    feature map (i.e.~anything in $S$) depends only on the radial coordinate.
    Similarly, any $h$ for which $\O h = 0$ must have 
    $\O h (r, \theta)  = \frac{1}{2\pi} \int_0^{2\pi}h(r, \theta) \dd \theta  = 0$ for any $r$,
    and $A$ consists entirely of such functions. For example, $r^{3} \cos \theta \in A$.
    We then recover $\inner{s}{h}_\mu = \frac{1}{2\pi}\int_\X s(r) h(r, \theta)\ee^{-r^2/2}r \dd r \dd \theta  = 0$ for any $s \in S$
    by integrating $h$ over $\theta$.
    Intuitively, one can think of the functions in $S$ as varying perpendicular to the flow
    of $\G$ on $\X = \R^2$ and so are preserved by it, while the functions in $A$ average to $0$ along this flow,
    see~\cref{fig:orbit}.
\end{example}

\begin{figure}[h]
    \centering
    \includegraphics[width=\figwidth]{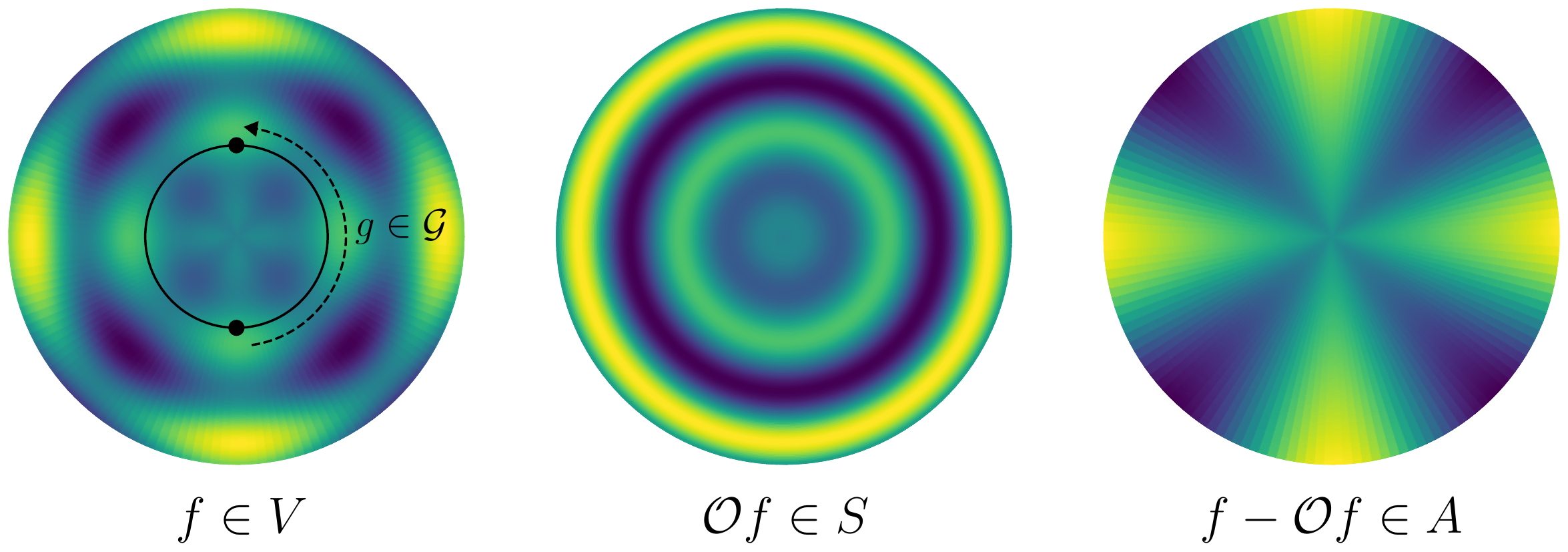}
    \vspace{-10pt}
    \caption{Example of a function decomposition.
    The figure shows $f(r, \theta) = r \cos{(r - 2\theta)} \cos{(r + 2\theta)}$
    decomposed into its symmetric and anti-symmetric parts in $V = S \oplus A$
    under the natural action of $\G = \SO(2)$ on $\R^2$.
    See~\Cref{example:rotations}. Best viewed in colour.}\label{fig:orbit}
\end{figure}

\section{Feature Averaging}\label{sec:fa}
We remarked earlier that $\O$ can be thought of as performing feature averaging.
Before our study of the generalisation of symmetric models, we use this interpretation to derive 
our first consequence of~\Cref{thm:fa}. 
We show that feature averaging can be viewed as 
solving a least squares problem in the space of features extractors $V$.
That is, feature averaging sends $f$ to $\bar{f}$, where $\bar{f}$ is the \emph{closest $\G$-invariant feature 
extractor to $f$}.

\newcommand{\thmleastsquares}{%
    Let $V$ be the space of all normalisable feature extractors as defined above.
    Define $S$ and $A$ as in~\Cref{thm:fa}.
    For any $f \in V$, feature averaging with $\O$ maps $f \mapsto \bar{f}$ where $\bar{f}$
    is the ($\mu$-a.e.) unique solution to the least-squares problem
    \[
        \bar{f} = \argmin_{s \in S} \norm{f - s}^2_\mu.
    \]%
}   
\begin{proposition}[Feature Averaging as a Least-Squares Problem]\label{thm:least-squares}
    \thmleastsquares{}
\end{proposition}
\begin{supplementary}
The proof of~\cref{thm:least-squares} is a straightforward exercise,
so we postpone it to~\cref{sec:proof-least-squares}.
\end{supplementary}

\begin{example}
    Consider again the setting of~\cref{example:rotations}.
    For simplicity, let $f(r, \theta)= f_{\text{rad}}(r)f_{\text{ang}}(\theta)$ be separable in polar coordinates.
    Notice that $\O f = c_f f_{\text{rad}}$ where $c_f  = \frac{1}{2\pi} \int_0^{2\pi} f_\text{ang}( \theta)\dd \theta $.
    Then for any $s \in S$ can calculate:
    \begin{align*}
        \norm{f - s}_\mu^2 
        &= \frac{1}{2\pi}\int_\X (f(r, \theta) - s(r))^2 \ee^{-r^2/2}r \dd r \dd \theta  \\
        &= \frac{1}{2\pi}\int_\X (f(r, \theta)- c_ff_{\text{rad}}(r))^2 \ee^{-r^2/2}r \dd r \dd \theta \\
        &\phantom{=}+ \frac{1}{2\pi}\int_\X(c_ff_{\text{rad}}(r) -s)^2 \ee^{-r^2/2}r \dd r \dd \theta 
    \end{align*}
    which is minimised by $s = c_f f_\text{rad}$, as predicted.
\end{example}

\subsection{Feature Averaging and Generalisation}\label{sec:fa-gen}
We end our discussion of feature averaging with an analysis of its impact on generalisation.
To do this we consider the reduction in the Rademacher complexity of a feature averaged class.

\paragraph{Rademacher Complexity}
Let $T = \{t_1, \dots, t_n \}$ be a collection of points of some space $\T$.
The empirical Rademacher complexity of a set $\F$ of functions $f: \T \to \R$
evaluated on $T$ is defined by
\[
    \Rad_T(\F)
    = \Esig\left[ \sup_{f \in \F}\abs{\frac1n  \sum_{i=1}^n \sigma_i f(t_i)}\right]
\]
where $\sigma_i \sim \Unif\{-1, 1\}$ for $i=1,\dots, n$ are Rademacher random variables over which
the expectation $\Esig$ is taken.
Let $\nu$ be a distribution on $\T$ and take $T \sim \nu^n$,
in which case the empirical Rademacher complexity $\Rad_T(\F)$ is a random quantity.
The Rademacher complexity $\ERad_n(\F)$ is defined by
taking an expectation over $T$:
$
    \ERad_n(\F) = \E[\Rad_T(\F)]
$.
The Rademacher complexity appears in the study of generalisation in statistical learning,
for instance see~\citet[Theorem 4.10 and Proposition 4.12]{wainwright2019high}.

\begin{proposition}\label{thm:fa-rademacher}
    Let $\G$ be a compact group acting measurably on a set $\T$.
    Let $\F$ be a class of functions $f: \T \to \R$ and define the symmetric and anti-symmetric classes
    $
        \overline{\F} = \{\O f : f \in \F\}
        $
        and
       $
       \F^\sperp = \{f - \O f: f \in \F\}
       $.
    Let $\nu$ be a distribution over $\T$ that is $\G$-invariant.
    Then the Rademacher complexity of the feature averaged class satisfies
   \[
        0 \le  \ERad_n (\F)  -  \ERad_n (\overline{\F})  \le  \ERad_n (\F^\sperp)
   \]
   where the expectations in the definition of $\ERad_n$ are taken over $t_i \sim \nu$ \iid.
\end{proposition}
\begin{supplementary}
See~\Cref{sec:proofs-appdx,sec:proof-fa-rademacher} for setup, proof and technicalities.
\end{supplementary}
\Cref{thm:fa-rademacher} says that the Rademacher complexity is reduced by feature averaging, but not by more than
the complexity of the anti-symmetric component of the class.
This can be thought of as quantifying the benefit of enforcing invariance by averaging in
terms of the extent to which the inductive bias is already present in the function class.
Although the form of this result is suggestive, it
does not imply a strict benefit.
We provide stronger results in the following section.

\section{Generalisation Benefit from Symmetric Models}\label{sec:lower-bounds}
In this section we apply~\Cref{thm:fa} to derive a strict (i.e.~non-zero) generalisation gap
between models that have and have not been specified to have the invariance/equivariance
that is present in the task.
We start with the following general result, which equates the generalisation gap
between any predictor and its closest equivariant function
to the norm of the anti-symmetric component of the predictor.

\begin{lemma}\label{lemma:gen-gap-antisym}
    Let $X \sim \mu$ where $\mu$ is a $\G$-invariant distribution on $\X$.
    Let $Y = f^*(X) + \xi \in \R^k$, where $\xi$ is a random element of $\R^k$
    that is
    independent of $X$ with zero mean and finite variance
    and $f^*: \X \to \R^k$ is $\G$-equivariant.
    Then, for any $f \in V$, the generalisation gap satisfies
    \if\submission1
    \begin{align*}
        \Delta(f, \q f) 
        &\coloneqq \E[\norm{Y - f(X)}_2^2] - \E[\norm{Y - \q f(X)}_2^2] \\
        &= \norm{f^\sperp}_\mu^2.
    \end{align*}
    \else
    \[
        \Delta(f, \q f) 
        \coloneqq \E[\norm{Y - f(X)}_2^2] - \E[\norm{Y - \q f(X)}_2^2] 
        = \norm{f^\sperp}_\mu^2.
    \]
    \fi
\end{lemma}
\begin{proof}
    By \Cref{thm:fa}, we can decompose $f = \bar{f} + f^\sperp$ into its
    symmetric $\bar{f} = \q f$ and anti-symmetric $f^\sperp = f - \q f$ parts.
    Since $\xi$ is independent of $X$ with zero mean and finite variance, 
    \[
        \Delta(f, \q f) = \E[\norm{Y - f(X)}_2^2] - \E[\norm{Y - \bar{f}(X)}_2^2] 
        = \E[\norm{f^*(X) - f(X)}_2^2] - \E[\norm{f^*(X) - \bar{f}(X)}_2^2].
    \]
    Using the decomposition of $f$,
    \begin{align*}
        \E[\norm{f^*(X) - f(X)}_2^2] - \E[\norm{f^*(X) - \bar{f}(X)}_2^2] 
        &= -2\E[\inner{f^*(X) - \bar{f}(X)}{f^\sperp(X)}] + \E[\norm{f^\sperp(X)}_2^2] \\
        &= \norm{f^\sperp}_\mu^2.
    \end{align*}
    Here we relied on $\E[\inner{f^*(X) - \bar{f}(X)}{f^\sperp(X)}] =
    \inner{f^* - \bar{f}}{f^\sperp}_\mu = 0$, which follows from $f^*$ being
    $\G$-equivariant and hence orthogonal to $f^\sperp$.
\end{proof}

\Cref{lemma:gen-gap-antisym} demonstrates the existence of barrier in the generalisation 
of any predictor on a problem that has a symmetry.
Notice that the barrier turns out to be the measure of how well the predictor encodes the symmetry.
Clearly, the only way of overcoming this is to set $f^\sperp = 0$ ($\mu$-a.e.), which from~\Cref{thm:fa}
equivalent to enforcing $\G$-equivariance in $f$ ($\mu$-a.e.).
\Cref{lemma:gen-gap-antisym} therefore provides a \emph{strict generalisation benefit for equivariant predictors}.

In a sense, $\q f$ is the archetypal equivariant predictor to which $f$ should be compared.
A trivial extension to~\cref{thm:least-squares} shows that $\q f$ is the closest equivariant 
predictor to $f$ and, more importantly,
if $h$ is a $\G$-equivariant predictor with smaller test
risk than $\q f$ then $\Delta(f, h) = \Delta(f, \q f) + \Delta(\q f, h) \ge \Delta(f, \q f)$
which cannot weaken our result.

Later in this section we will use~\Cref{lemma:gen-gap-antisym}
to explicitly calculate the generalisation gap for invariant/equivariant linear models.
We will see that it displays a natural relationship
between the number of training examples and the dimension of the space of
anti-symmetric models $A$, which is a property of the group action.
Intuitively, the model needs enough examples to learn to be orthogonal to $A$.

This result also has a useful theoretical implication for test-time data
augmentation, which is commonly used to increase test accuracy
\cite{Simonyan15, Szegedy15, He16}. Test-time augmentation consists of
averaging the output of a learned function $f$ over random transformations of
the same input when making predictions at test-time. When the transformations
belong to a group $\G$ and are sampled from its Haar measure,
test-time averaging can be viewed as a Monte Carlo estimate of $\O f$.
\Cref{lemma:gen-gap-antisym} then shows that test-time
averaging is beneficial for generalisation when the target function is itself
$\G$-invariant, regardless of the learned function $f$.

\subsection{Regression with Invariant Target}
Let $\X = \R^d$ with the Euclidean inner product and $\Y = \R$.
Consider linear regression with the squared-error loss $\ell(y, y') = (y - y')^2$.
Let $\G$ be a compact group that acts on $\X$ via an orthogonal representation $\phi: \G \to O(d)$
and let $X \sim \mu$ where $\mu$ is now an arbitrary $\G$-invariant probability distribution
on $\X$ with $\Sigma \coloneqq \E[X X^\top]$ finite and positive definite.%
\footnote{If $\Sigma$ is only positive semi-definite then the developments are similar.
We assume $\Sigma > 0$ for simplicity.}
We consider linear predictors $h_w: \X \to \Y$ with $h_w(x) = w^\top x$ where $w \in \X$.
Define the space of all linear predictors $\vlin = \{h_w : w \in \X\}$ which is a subspace of $V$.
Notice that $\vlin$ is closed under $\O$: for any $x \in \X$
\if\submission1
\begin{align*}
    \O h_w (x) 
	&= \int_\G  h_w(gx) \dd\lambda(g) \\
	&= \int_\G w^\top \phi(g) x \dd\lambda(g)\\
	&= \left(\int_\G \phi(g^{-1}) w\dd\lambda(g)\right)^\top x \\
	&= h_{\proj(w)}(x)
\end{align*}
\else
\[
    \O h_w (x) = \int_\G  h_w(gx) \dd\lambda(g) = \int_\G w^\top \phi(g) x \dd\lambda(g)
        = \left(\int_\G \phi(g^{-1}) w\dd\lambda(g)\right)^\top x = h_{\proj(w)}(x)
\]
\fi
where in the last line we substituted $g \mapsto g^{-1}$ and defined the linear map $\proj: \R^d \to \R^d$
by $\proj(w) = \int_\G  \phi(g) w\dd \lambda(g) $.%
\footnote{Since $\G$ is compact it is unimodular
and this change of variables is valid, e.g.~see~\citet[Corollary 2.28]{folland2016course} and adjacent results.}
We also have
\[
    \inner{h_a}{h_b}_\mu = \int_\X a^\top xx^\top b\dd\mu(x)  = a^\top \Sigma b
\]
and we denote the induced inner product on $\X$ by $\inner{a}{b}_\Sigma \coloneqq a^\top \Sigma b$
and the corresponding norm by $\norm{\cdot}_\Sigma$.
Since $\vlin$ is closed under $\O$ we can apply~\Cref{thm:fa} to decompose $\vlin = S \oplus A$
where the orthogonality is with respect to $\inner{\cdot}{\cdot}_\mu$.
It follows that we can write any $h_w \in \vlin$ as
\[
    h_w = \overline{h_w} + h_w^\sperp
\]
where we have shown that there must exist $\bar{w}, w^\sperp \in \X$ with $\inner{\bar{w}}{w^\sperp}_\Sigma = 0$
such that $\overline{h_w} = h_{\bar{w}}$ and $h_w^\sperp = h_{w^\sperp}$.
By choosing a basis for $\X$, there is a natural bijection $\X \to \vlin$ where $w \mapsto h_w$.
Using this identification, we abuse notation slightly and write $\X = S \oplus A$ 
to represent the induced structure on $\X$.

Suppose examples are labelled by a target function $h_\theta \in \vlin$ that is $\G$-invariant.
Let $X \sim \mu$ and $Y = \theta^\top X + \xi$ where $\xi$ is independent of $X$, has mean 0 and finite variance.
Recall the definition of the \emph{generalisation gap} between predictors as the 
difference in their test errors.
We study the generalisation gap $\Delta(h_w, h_{\bar{w}})$ between predictors $h_w$ and $h_{\bar{w}}$
defined above.
\Cref{lemma:gen-gap-antisym} gives
$
    \Delta(h_w, h_{\bar{w}}) = \norm{h_{w^\sperp}}_\mu^2 = \norm{w^\sperp}_\Sigma^2
    $.
In~\Cref{thm:invariant-regression} we calculate this quantity
where $w$ is the minimum-norm least-squares estimator and $\bar{w} =\Phi_\G(w)$.
To the best of our knowledge, this is the first result to specify exactly 
the generalisation benefit for invariant models.

\begin{theorem}\label{thm:invariant-regression}
    Let $\X = \R^d$, $\Y = \R$ and let $\G$ be a compact group with an orthogonal representation
    $\phi$ on $\X$. Let $X \sim \n(0, \sigma_X^2 I)$ 
    and $Y = h_\theta(X) + \xi$ where $h_\theta(x) = \theta^\top x$ is $\G$-invariant
    with $\theta \in \R^d$ and $\xi$ has mean $0$, variance $\sigma_\xi^2 < \infty$ and is independent of $X$.
    Let $w$ be the least-squares estimate of $\theta$ from \iid~examples $\{(X_i, Y_i): i=1, \dots, n\}$ and let
    $A$ be the orthogonal complement of the subspace of $\G$-invariant linear predictors (c.f.~\Cref{thm:fa}).
    \begin{itemize}
        \item If $n > d + 1$ then the generalisation gap is
            \[
                \E[\Delta(h_w, h_{\bar{w}})] = \sigma_\xi^2\frac{\dim A}{n - d - 1}.
            \]
        \item At the interpolation threshold $n \in [d-1, d+1]$, if $h_w$ is not $\G$-invariant then 
            the generalisation gap diverges to $\infty$.
        \item If $n < d- 1$ the generalisation gap is 
            \if\submission1
            \begin{align*}
                &\E[\Delta(h_w, h_{\bar{w}})]  \\ 
                &= \dim A \left( \frac{\sigma_X^2 \norm{\theta}_2^2\, n(d-n)}{d(d-1)(d+2)}
                + \frac{\sigma_\xi^2 \, n}{d(d - n - 1)} \right) .
            \end{align*}
            \else
            \[
                \E[\Delta(h_w, h_{\bar{w}})] = 
                \dim A \left( \sigma_X^2 \norm{\theta}_2^2\frac{n(d-n)}{d(d-1)(d+2)}
                + \sigma_\xi^2\frac{n}{d(d - n - 1)} \right) .
            \]
            \fi
    \end{itemize}
\end{theorem}

In each case, the generalisation gap has a term of the form $\sigma_\xi^2 r(n, d) \dim A$ that arises 
due to the noise in the target distribution. 
In the overparameterised setting $d > n+1$ there is an additional term (the first)
that represents the generalisation gap in the noiseless setting $\xi \eqas 0$.
This term is the error in the least-squares estimate of $\theta$ in the noiseless problem,
which of course vanishes in the fully determined case $n > d + 1$.
In addition, the divergence at the so called interpolation threshold $n\approx d$ is consistent with
the literature on double descent~\cite{hastie2019surprises}.

Notice the central role of $\dim A$ in~\Cref{thm:invariant-regression}.
This quantity is a property of the group action as it describes the codimension of the set of invariant models.
The generalisation gap is then dictated by how `significant' the symmetry is to the problem.
\if\submission1
We give two examples that represent extremal cases of this `significance'.
\else
Before turning to the proof, we give two examples that represent extremal cases of this `significance'.
\fi

\begin{example}[Permutations, $\dim A = d - 1$]
    Let $S_d$ act on $\X = \R^d$ by permutation of the coordinates, so $(\phi(\rho)w)_i = w_{\rho(i)}$ for $\rho \in S_d$.
    Observe that, since the Haar measure $\lambda$ is uniform on $S_d$, for any $i=1, \dots, d$
    \[
        \Phi_{S_d}(w)_i = \frac{1}{d!} \sum_{\rho \in S_d} w_{\rho(i)} = \frac{1}{d} \sum_j w_j
    \]
    so $S$ is the one dimensional subspace $\{t (1, \dots, 1)^\top: t \in \R \}$.
    Since $\X = S \oplus A$ we get $\dim A = d-1$. 
\end{example}

\begin{example}[Reflection, $\dim A = 1$]
    Let $C_2$ be the cyclic group of order 2 and let it act on $\X = \R^d$ by reflection
    in the first coordinate. $A$ is then the subspace consisting of $w$ such that for any $j=2, \dots, d$
    \[
        \Phi_{C_2}( w)_j = \frac{1}{\abs{C_2}}\sum_{g\in C_2} (\phi(g)w)_j = w_j = 0
    \]
    since the action fixes all coordinates apart from the first. 
    Hence $A = \{t(1, 0, \dots, 0)^\top: t \in \R\}$.
\end{example}

\begin{proof}[Proof of~\Cref{thm:invariant-regression}]
    Note that $X$ is $\G$-invariant for any $\G$ since the representation $\phi$ is orthogonal.
    We have seen above that the space of linear maps $\vlin = \{h_w : w \in \R^d\}$ is closed under $\O$,
    so by~\Cref{thm:fa} we can write $\vlin = S \oplus A$.
    Let $\proj^\sperp = I - \proj$, which is the orthogonal projection onto the subspace $A$.
    By isotropy of $X$ we have
    \[
        \Delta(h_w, h_{\bar{w}}) = \sigma_X^2 \norm{w^\sperp}_2^2
    \]
    for any $w \in \R^d$, where $w^\sperp = \proj^\sperp(w)$.
    The proof consists of calculating this quantity in the case that $w$ is the least-squares estimator.
        
    Let $\x \in \R^{n \times d}$ and $\y \in \R^n$ correspond to row-stacked training examples
    drawn \iid~as in the statement, so $\x_{ij} = (X_i)_j$ and $\y_i = Y_i$.
    Similarly, set $\bm{\xi} = \x\theta - \y$.
    The least squares estimate is the minimum norm solution of $ \argmin_{u \in \R^d} \norm{\y - \x u}_2^2$,
    i.e.
    \begin{equation}
        w = (\x^\top\x)^+ \x^\top \x \theta + (\x^\top \x)^+ \x^\top \bm{\xi} \label{eq:soln}
    \end{equation}
    where $(\cdot)^+$ denotes the Moore-Penrose pseudo-inverse.
    Define $P_{E} = (\x^\top \x)^+ \x^\top \x$, which is an orthogonal projection onto $E$, the
    rank of $\x^\top \x$ (this can be seen by diagonalising).

    We first calculate $\E[\norm{w^\sperp}_2^2 \vert{} \x]$ where $w^\sperp = \proj^\sperp(w)$.
    The contribution from the first term of~\cref{eq:soln} is
    $
        \norm{\proj^\sperp(P_E \theta)}_2^2
        $
    the cross term vanishes using $\xi \indep X$ and $\E[\xi] = 0$ and the contribution from the second term of~\cref{eq:soln}
    is $
        \E[\norm{\proj^\sperp((\x^\top \x)^+ \x^\top \bm{\xi}) }_2^2 \vert{} \x] 
        $. Notice that $\proj^\sperp$ is just projection matrix and so is idempotent, hence (briefly writing it without the parenthesis
        to emphasis matrix interpretation)
    \begin{align*}
        \E[\norm{\proj^\sperp((\x^\top \x)^+ \x^\top \bm{\xi}) }_2^2 \vert{} \x] 
        &=  \E[  \tr( \bm{\xi}^\top \x  (\x^\top \x)^+\proj^\sperp(\x^\top \x)^+ \x^\top \bm{\xi}  )\vert{} \x] \\
        &=  \tr( \x  (\x^\top \x)^+\proj^\sperp(\x^\top \x)^+ \x^\top \E[\bm{\xi}  \bm{\xi}^\top] ) \\
        &=  \sigma_\xi^2\tr( \proj^\sperp(\x^\top \x)^+).
    \end{align*}
    We have obtained
    \[
        \E[\norm{w^\sperp}_2^2\vert{} \x]
        = \norm{\proj^\sperp (P_E \theta )}_2^2 
        +  \sigma_\xi^2\tr( \proj^\sperp((\x^\top \x)^+))
    \]
    and conclude by taking expectations, treating each term separately.
    \paragraph{First Term}
    If $n \ge d$ then $\dim E = d$ with probability $1$, so the first term vanishes almost surely.
    We treat the $n < d$ case using Einstein notation, in which repeated indices are implicitly summed over.
    In components, recalling that $\proj^\sperp$ is a matrix,
    \[
        \E[ \norm{\proj^\sperp (P_E \theta )}_2^2 ] = \proj^\sperp_{fa}\proj^\sperp_{fc}\E[P_E \otimes P_E]_{abce} \theta_b \theta_e
    \]
    and applying~\cref{lemma:proj-variance} we get
    \begin{align*}
        \E[ \norm{\proj^\sperp (P_E \theta )}_2^2 ] 
        &= 
        \frac{n(d-n)}{d(d-1)(d+2)}
        \left(
            \proj^\sperp_{fa}\proj^\sperp_{fa} \theta_b \theta_b
            + \proj^\sperp_{fa}\proj^\sperp_{fb} \theta_b \theta_a
        \right)\\
        &\phantom{=}+ \frac{n(d-n) + n(n-1)(d+2)}{d(d-1)(d+2)} \proj^\sperp_{fa}\proj^\sperp_{fc}\theta_a \theta_c\\
        &= \norm{\theta}_2^2\dim A\frac{n(d-n)}{d(d-1)(d+2)}
    \end{align*}
    where we have used that $\proj^\sperp(\theta) = 0$ and $\fnorm{\proj^\sperp}^2 = \dim A$.

    \paragraph{Second Term}
    By linearity,
    \[
        \E[\tr( \proj^\sperp((\x^\top \x)^+))] = \tr( \proj^\sperp(\E[(\x^\top \x)^+])).
    \]
    Then~\Cref{lemma:expected-inv-wishart-singular,lemma:expected-inv-wishart}
    give $\E[(\x^\top\x)^+] = \sigma_X^{-2} r(n,d)I_d$ where 
    \[
        r(n, d) = \begin{cases}
            {\frac{n}{d(d - n - 1)}} & n < d- 1\\
            (n - d - 1)^{-1} & n > d + 1 \\ 
            \infty & \text{otherwise}
        \end{cases}.
    \]
    When $n \in [d-1, d+1]$ it is well known that the expectation diverges, 
    see~\Cref{sec:useful-facts}. Hence
    \[
        \E[\tr( \proj^\sperp((\x^\top \x)^+))] = \sigma_X^{-2} r(n, d)\dim A.
    \]
\end{proof}

\subsection{Regression with Equivariant Target}\label{sec:equivariant-regression}
One can apply the same construction to equivariant models.
Assume the same setup, but now let $\Y = \R^k$ with the Euclidean inner product and 
let the space of predictors be
$\wlin = \{f_W: \R^d \to \R^k, \,\, f_W(x) = W^\top x : W \in \R^{d \times k}\}$.
We consider linear regression with the squared-error loss $\ell(y, y') = \norm{y - y'}_2^2$.
Let $\psi$ be an orthogonal representation of $\G$ on $\Y$.
We define the linear map, which we call the intertwining average,
$\twine :\R^{d \times k} \to \R^{d \times k} $ by%
    \footnote{The reader may have noticed that we define $\twine$ backwards, in the sense that its
        image contains maps that are equivariant in the direction $\psi \to \phi$.
        This is because of the transpose in the linear model,
        which is there for consistency with the $k=1$ invariance case.
    This choice is arbitrary and gives no loss in generality.}
\[
    \twine(W) = \int_\G  \phi(g) W \psi(g^{-1})\dd \lambda(g).
\]
Similarly, define the intertwining complement as $\twine^\sperp :\R^{d \times k} \to \R^{d \times k} $ by
$
    \twine^\sperp(W) = W - \twine(W)
$.
We establish the following results, which are in fact generalisations of the invariant case.
In the proofs we will leverage the expression of $\twine$ as a $4$-tensor with components
$
    \twine_{abce} = \int_\G  \phi(g)_{ac}\psi(g)_{be}\dd \lambda(g)
    $ where $a, c = 1, \dots d$ and $b, e = 1, \dots, k$.%
\footnote{If necessary, the reader can see%
\if\submission1
the supplementary material
\else
~\cref{sec:cpts}
\fi
for a derivation.}
\begin{proposition}\label{prop:closure}
    For any $f_W \in \wlin$, $\q f_W = f_{\twine( W)}$ and hence $\wlin$ is closed under $\q$.
\end{proposition}
\begin{proof}
    Let $f_W(x) = W^\top x$ with $W \in \R^{d \times k}$, then using orthogonality and unimodularity 
    \[
        \q f_W (x) 
        = \int_\G  \psi(g^{-1}) W^\top \phi(g) x \dd \lambda(g)
        = \left(\int_\G  \phi(g) W \psi(g^{-1}) \dd \lambda(g)\right)^\top x
        = \twine(W)^\top x .
    \]
\end{proof}

\begin{proposition}\label{prop:inner-product}
    The inner product on $\wlin$ satisfies, for any $f_{W_1}, f_{W_2} \in \wlin$,
    \[
        \inner{f_{W_1}}{f_{W_2}}_\mu = \tr( W_1^\top \Sigma W_2)
    \]
    where $\Sigma = \E[X X^\top]$ and $X \sim \mu$.
\end{proposition}
\begin{proof}
    \begin{align*}
        \inner{f_{W_1}}{f_{W_2}}_\mu
        &= \int_\X  \left( W_1^\top x\right)^\top W_2^\top x \dd \mu (x)\\
        &= \int_\X  x^\top W_1 W_2^\top x \dd \mu (x)\\
        &= \int_\X  \tr(x^\top W_1 W_2^\top x) \dd \mu (x)\\
        &= \tr( W_1^\top \Sigma W_2)
    \end{align*}
\end{proof}

\Cref{prop:closure} allows us to apply~\Cref{thm:fa} to write $\wlin = S \oplus A$, so for any $f_W \in \wlin$
there exists $\overline{f_{W}} \in S$ and $f_W^\sperp \in A$ with $\inner{\overline{f_W}}{f_W^\sperp}_\mu = 0$.
The corresponding parameters $\overline{W} = \twine(W)$ and $W^\sperp = \twine^\sperp(W)$ must therefore satisfy
$\tr(\overline{W}^\top \Sigma W^\sperp) = 0$, with $\Sigma$ defined as in~\Cref{prop:inner-product}.
Repeating our abuse of notation, we identify 
$\R^{d \times k} = S \oplus A$ with $S = \twine(\R^{d \times k})$ and $A$ its orthogonal
complement with respect to the induced inner product.
\begin{proposition}\label{prop:generalisation-gap}
    Let $X \sim \mu$ and let $\xi$ a random element of $\R^k$ that is independent of
    $X$ with $\E[\xi] = 0$ and finite variance.
    Set $Y = f_\Theta (X) + \xi$ where $f_\Theta$ is $\G$-equivariant.
    For any $f_W \in \wlin$, the generalisation gap satisfies
    \if\submission1
    \begin{align*}
        \Delta(f_W, f_{\overline{W}}) 
        &\coloneqq \E[\norm{Y - f_W(X)}_2^2] - \E[\norm{Y - f_{\overline{W}}(X)}_2^2] \\
        &= \fnorm{\Sigma^{1/2} W^\sperp}^2
    \end{align*}
    \else
    \[
        \Delta(f_W, f_{\overline{W}}) \coloneqq \E[\norm{Y - f_W(X)}_2^2] - \E[\norm{Y - f_{\overline{W}}(X)}_2^2] 
        = \fnorm{\Sigma^{1/2} W^\sperp}^2
    \]
    \fi
    where $\overline{W} =\twine(W)$, $W^\sperp = \twine^\sperp(W)$
    and $\Sigma = \E[XX^\top]$.
\end{proposition}
\begin{proof}
    Recall that $W = \overline{W} + W^\sperp$ and that these 
    satisfy $\tr(\overline{W}\Sigma W^\sperp) =0$ from the above.
    Then, using~\Cref{lemma:gen-gap-antisym} and~\cref{prop:inner-product},
    \[
        \Delta(f_W, f_{\overline{W}})  
        = \norm{f_{W^\perp}}_\mu^2 
        = \tr((W^\sperp)^\top \Sigma W^\sperp)
        = \fnorm{\Sigma^{1/2} W^\sperp}^2
    \]
\end{proof}

Having followed the same path as the previous section,
we provide a characterisation of the generalisation benefit of equivariance.
In the same fashion, we compare the least-squares estimate $W$ with its equivariant
version $\overline{W} = \twine(W)$. As we explained at the beginning of the section, the choice of
$\overline{W} = \twine(W)$ is natural and costs us nothing.
\newcommand{\thmequivariantregression}{%
    Let $\X = \R^d$, $\Y = \R^k$ and let $\G$ be a compact group with orthogonal representations
    $\phi$ on $\X$ and $\psi$ on $\Y$. Let $X \sim \n(0, \sigma_X^2 I_d)$ 
    and $Y = h_\Theta(X) + \xi$ where $h_\Theta(x) = \Theta^\top x$ is $\G$-equivariant and $\Theta \in \R^{d \times k}$.
    Assume $\xi$ is a random element of $\R^k$, independent of $X$, with mean $0$ and $\E[\xi \xi^\top ] = \sigma_\xi^2 I_k < \infty$.
    Let $W$ be the least-squares estimate of $\Theta$ from $n$ \iid~examples
    $\{(X_i, Y_i): i=1, \dots, n\}$ 
    and let $\binner{\chi_\psi}{\chi_\phi} = \int_{\G} \chi_\psi(g) \chi_\phi(g)\dd \lambda(g) $ denote the scalar product of the
    characters of the representations of $\G$.
    \begin{itemize}
        \item     If $n > d + 1$ the generalisation gap is 
    \[
        \E[\Delta(f_W, f_{\overline{W}})] =
        \sigma_\xi^2 \frac{ dk - \binner{\chi_\psi}{\chi_\phi}}{n - d - 1}. 
    \]
    \item At the interpolation threshold $n \in [d-1, d+1]$, if $f_W$ is not $\G$-equivariant then
    the generalisation gap diverges to $\infty$.
    \item If $n < d- 1$ then the generalisation gap is 
        \if\submission1
        \begin{align*}
        &\E[\Delta(f_W, f_{\overline{W}})] =\\
        &\sigma_X^2 \frac{n(d-n)}{d(d-1)(d+2)} \left( (d+1) \fnorm{\Theta}^2- \tr( J_\G \Theta^\top \Theta) \right) \\
        &+ \sigma_\xi^2 \frac{n( dk - \binner{\chi_\psi}{\chi_\phi})}{d (d - n - 1)}
        \end{align*}
        \else
    \[
        \E[\Delta(f_W, f_{\overline{W}})] =
        \sigma_X^2 \frac{n(d-n)}{d(d-1)(d+2)} \left( (d+1) \fnorm{\Theta}^2- \tr( J_\G \Theta^\top \Theta) \right) 
        + \sigma_\xi^2 \frac{n( dk - \binner{\chi_\psi}{\chi_\phi})}{d (d - n - 1)}
    \]
    \fi
    where each term is non-negative and 
    $J_\G \in \R^{k \times k}$ is given by
    \[
        J_\G = \int_\G (\chi_\phi(g) \psi(g) + \psi(g^2)) \dd \lambda(g).
    \]
    \end{itemize}
}
\begin{theorem}\label{thm:equivariant-regression}
\thmequivariantregression
\end{theorem}
\begin{supplementary}
The proof of~\Cref{thm:equivariant-regression} is longer than for~\Cref{thm:invariant-regression}
but follows the same scheme, so we defer it to~\Cref{sec:equivariant-proof}.
\end{supplementary}

\Cref{thm:equivariant-regression} is a direct generalisation of~\Cref{thm:invariant-regression}. 
As we remarked in the introduction, $dk - \binner{\chi_\psi}{\chi_\phi}$ plays the role of $\dim A$ 
in~\Cref{thm:invariant-regression} and is a measure of the significance of the symmetry to the problem. 
The dimension of $\wlin$ is $dk$, while $\binner{\chi_\psi}{\chi_\phi}$ is the dimension of the space of equivariant maps.
In our notation $\binner{\chi_\psi}{\chi_\phi} = \dim S$.

Just as with~\cref{thm:invariant-regression}, there is an additional term (the first)
in the overparameterised case $d > n + 1$ that represents the estimation in a noiseless setting $\xi \eqas 0$.
Notice that if $k=1$ and $\psi$ is trivial we find
\[
    J_\G = \int_\G \chi_\phi(g) \dd \lambda(g)  + 1 =  \binner{\chi_\phi}{1} + 1 = \dim S + 1
\]
which confirms that~\cref{thm:equivariant-regression} reduces exactly to~\cref{thm:invariant-regression}.

Interestingly, the first term in the $d > n+1$ case can be made independent of $\psi$, since the equivariance of
$h_\Theta$ implies
\[
    \tr(J_\G \Theta^\top \Theta) = \tr(\Theta^\top J_\phi \Theta)
\]
where
\[
    J_\phi = \int_\G (\chi_\phi(g) \phi(g) + \phi(g^2)) \dd \lambda(g).
\]

Finally, we remark that~\cref{thm:equivariant-regression} is possible for more general 
probability distributions on $X$.
For instance, it sufficient that the distribution is absolutely continuous with
respect to the Lebesgue measure, has finite variance and is $O(d)$ invariant.
The final condition implies the existence of a scalar $r_{n}$
such that $\E[(\x^\top \x)^+] = r_{n} I_d$
where $\x \in \R^{n \times d}$ are the row-stacked training inputs as defined in the proof.

\section{Neural Networks}\label{sec:neural-networks}
In this section we discuss how the insights of this paper apply to neural networks
and raise some open questions for future work.
Let $F: \R^d \to \R^k$ be a feedforward neural network with $L$ layers,
layer widths $\kappa_i$ $i=1, \dots, L$ and weights $W^i \in \R^{\kappa_{i} \times \kappa_{i+1}}$
for $i=1, \dots, L$ where $\kappa_1 = d$ and $\kappa_{L+1} = k$. We will assume $F$ has the form
\begin{equation}\label{eq:nn}
    F(x) = W^L \sigma(W^{L-1}\sigma(\dots \sigma(W^{1} x)\dots))
\end{equation}
where $\sigma$ is an element-wise non-linearity.

\subsection{Invariant and Equivariant Networks}\label{sec:invariant-networks}
The standard method for engineering neural networks to be invariant/equivariant
to the action of a finite group on the inputs is weight tying. 
This method has been around for a while~\cite{wood1996representation}
but has come to recent attention via~\citet{ravanbakhsh2017equivariance}.
We will briefly describe this approach, its connections 
to~\Cref{thm:invariant-regression,thm:equivariant-regression} and how the ideas of
this paper can be used to find new algorithms for both enforced and learned invariance/equivariance.
We leave analyses of these suggested approaches to future work.

The methods of~\citet{wood1996representation,ravanbakhsh2017equivariance} can be described as follows.
Let $\Gfin$ be a finite group.
For each $i=2, \dots, L+1$, the user chooses a matrix representation $\psi_i:\Gfin \to \GL_{k_i}(\R)$
of $\Gfin$ that acts on the inputs for each layer $i=2, \dots, L$
and on the outputs of the network when $i=L+1$.%
\footnote{$\psi_1$ is the representation on the inputs, which we consider as an
aspect of the task and not a design choice.}
For $i=2, \dots, L$, these representations must be chosen
such that they commute with the activation function 
\begin{equation}\label{eq:activation}
  \sigma(\psi_i(g) \cdot) = \psi_i(g) \sigma(\cdot)  
\end{equation}
$\forall g \in \Gfin$.%
\footnote{This condition is somewhat restrictive, but note that a permutation representation will commute with
any element-wise non-linearity.}
One then chooses weights for the network such that at each layer and $\forall g \in \Gfin$
\begin{equation}\label{eq:intertwine}
    W^i\psi_i(g) = \psi_{i+1}(g)W^i.
\end{equation}
By induction on the layers, satisfying~\cref{eq:activation,eq:intertwine}
ensures that the network is $\Gfin$-equivariant.
Invariance occurs when $\psi_{L+1}$ is the trivial representation.

The condition in~\cref{eq:intertwine} can be phrased as saying that that $W^i$ belongs
to the space of \emph{intertwiners} of the representations $\psi_i$ and $\psi_{i+1}$.
By denoting the space of all weight matrices in layer $i$ as $U = \R^{\kappa_i \times \kappa_{i+1}}$,
the space of intertwiners is immediately recognisable as $S = \twine(U)$ from~\Cref{thm:fa}.

Typically, the practitioner will hand-engineer the structure of weight matrices to belong to 
the correct intertwiner space.
In the following sections we will propose alternative procedures that build naturally on the ideas
of this paper.
Moreover, as a benefit of our framework, these new approaches extend weight-tying to 
any compact group that admits a finite-dimensional, real representation.
We end this section with a bound on the sample complexity of invariant networks, 
which follows from~\citet{bartlett2019nearly}. Similar results are possible for 
different activation functions.

\begin{lemma}\label{lemma:inavariant-vc}
    Let $\G$ be a compact group with layer-wise representations as described.
    Let $F: \R^d \to \R$ be a $\G$-invariant neural network with ReLU activation and
    weights that intertwine the representations.
    Let $\H$ be the class of all functions realisable by this network.
    Then 
    \[
        \vc{\H} \le L +  \frac{1}{2} \alpha(F) L(L+1)\max_{1 \le i \le L} \binner{\chi_{i}}{\chi_{i+1}}
    \]
    where $\alpha(F) = \log_2\left(4 \ee \log_2\left(\sum_{i=1}^L 2\ee i \kappa_i\right) \sum_{i=1}^L i\kappa_i \right)$.
\end{lemma}
\begin{proof}
    For a ReLU network $t_i$ independent parameters at each layer we have
    $
        \vc{\H} \le L + \alpha(F) \sum_{i=1}^L (L-i+1) t_i,
        $
    which follows by application of~\citet[Theorem 7]{bartlett2019nearly}.
    The proof of~\citet[Theorem 7]{bartlett2019nearly} depends on
    the representation of the network in terms of piecewise polynomials of bounded degree.
    Observe that since
    weights are tied only \emph{within} layers (so weights in different layers can vary independently)
    and the activation is ReLU, there is no increase in the degree of said polynomials
    from weight-tying and the proof given in~\citet{bartlett2019nearly} applies in our case.
    The condition~\cref{eq:intertwine} insists that the weight matrix $W^i$ belongs to the intertwiner space
    $\Hom_{\G}(\R^{\kappa_{i-1}\times\kappa_{i}}, \R^{\kappa_i\times\kappa_{i+1}})$.
    The number of independent parameters at each layer is at most the dimension of this space.
    The conclusion follows by simple algebra and the relation of the dimension to the characters as
    given above.
\end{proof}

\begin{example}[Permutation invariant networks]
    Permutation invariant networks (and permutation invariant functions more generally) are 
    studied in~\citet{wood1996representation,zaheer2017deep,bloem2020probabilistic} and many other works, 
    see references therein.
    In particular, multiple authors have given the form of a permutation equivariant weight
    matrix as 
    $
    W = \lambda I + \gamma \bm{1}\bm{1}^\top
            $
    for scalars $\alpha, \beta \in \R$ and with $\bm{1} = (1, \dots, 1)^\top$.
    Consider an $L$-layer ReLU network with, for simplicity, widths $\kappa_i = d$ $\forall i$.
    Let $\H$ be the class of all functions realisable by this network,
    then $\vc{\H} = O(L^2 \log(Ld \log(Ld)))$.
\end{example}

\subsection{Projected Gradients}
As we have seen, provided that the activation function satisfies~\cref{eq:activation},
specifying the weight matrices to intertwine between layer-wise representations
is sufficient to ensure equivariance in a neural network.
We have also seen from~\Cref{sec:equivariant-regression} that it is possible to 
project any weight matrix into an intertwiner space using $\twine$.
For each layer $l$ of the network we have a linear map $\twine^l$,
which is a $4$-tensor with components
$\twine^l_{abce} = \int_\G  \psi_{l+1}(g)_{ac}\psi_{l}(g)_{be}\dd \lambda(g)$.

The tensors $\{\twine^l: l=1, \dots, L\}$ depend only on the representations and so can be computed before training.
One can therefore obtain invariant/equivariant networks by a form of projected gradient descent.
Explicitly, with loss $\ell$ and learning rate $\eta$, the update rule for 
the $l$\textsuperscript{th} layer is
\begin{align*}
    &\widetilde{W}^l(t+1) = W^l(t) - \eta \grad_{W^l} \ell(W^1(t), \dots, W^L(t))\\
    &W^l(t+1) = \twine^l\left(\widetilde{W}^l(t+1)\right).
\end{align*}
If~\cref{eq:activation} holds the network will be exactly invariant/equivariant after any iteration.

\subsection{Regularisation for Equivariance}
We have seen from~\Cref{thm:fa,sec:equivariant-regression} that any weight matrix can be written
$W = \overline{W} + W^\sperp$ where $\overline{W} = \twine(W)$ belongs to an intertwiner space (so is equivariant)
and $W^\sperp = \twine^\sperp(W)$ belongs to an orthogonal space that parametrises the anti-symmetric linear maps.
This suggests a method of learned invariance/equivariance by using a regularisation term
of the form $\sum_{l=1}^L \fnorm{{\twine^l}^\sperp(W^l)}^2$.
Where the $4$-tensor ${\twine^l}^\sperp$ has 
components ${\twine^l}^\sperp_{abce} = \delta_{ac}\delta_{be} - \twine^l_{abce}$
and can be computed before training.
If ${\twine^l}^\sperp(W^l) = 0$ for $l=1, \dots, L$ and the activation function satisfies~\cref{eq:activation},
then the resulting network will be exactly $\G$-invariant/equivariant.
This method could also allow for learned/approximate invariance.
Indeed,~\Cref{prop:regularisation} suggests $\fnorm{{\twine^l}^\sperp(W)}^2$
as a measure of the layer-wise invariance/equivariance of the network.
\begin{proposition}\label{prop:regularisation}
    Let $\G$ be a compact group.
    Let $f_W: \R^d \to \R^k$ with $f_W(x) = \sigma(Wx)$ be a single neural network layer with $C$-Lipschitz, 
    element-wise activation $\sigma$.
    Let $\phi: \G \to O(d)$ and $\psi: \G \to O(k)$ be orthogonal representations of $\G$ on the input
    and output spaces respectively and assume that $\psi$ commutes with $\sigma$ as in~\cref{eq:activation}.
    Let $X \in \R^d$, $X \sim \n(0, I_d)$. We can consider the network as belonging to $V$ from~\Cref{sec:decomposition}
    with $\mu = \n(0, I_d)$. Write $V = S \oplus A$, where $S$ contains the equivariant functions in $V$, then
    \[
        \inf_{s \in S} \E[\norm{f_W(X) - s(X)}_2^2] \le 2C^2\fnorm{W^\sperp}^2.
    \]
\end{proposition}
\begin{proof}
    First note that the infimum clearly exists, since the left hand side vanishes when $W$ intertwines
    $\phi$ and $\psi$.
    Recognise that in the notation of~\Cref{sec:decomposition} we can write
    $\norm{a - b}_\mu^2 = \E[\norm{a(X) - b(X)}_2^2]$.
    By applying the proof of~\Cref{thm:least-squares} to $\q$ we get
    $\inf_{s \in S}\E[\norm{f_W(X) - s(X)}^2] = \inf_{s \in S}\norm{f_W - s}_\mu^2 = \norm{f_W - \q f_W}_\mu^2$.
    Then recalling the definition of $\q$ we have
    \begin{align*}
        \norm{\q f_W(x) - f_W(x) }_2^2
        &=  \lnorm{\int_\G  \psi(g^{-1}) f_W(\phi(g) x) - f_W(x)}_2^2 \dd \lambda(g)\\ 
        &\le  \int_\G  \norm{\psi(g^{-1}) f_W(\phi(g) x) - f_W(x)}_2^2 \dd \lambda(g)\\ 
        &=  \int_\G  \norm{\psi(g^{-1}) \sigma(W\phi(g) x) - \sigma(Wx)}_2^2 \dd \lambda(g)\\ 
        &=  \int_\G  \norm{\sigma(\overline{W}x + \psi(g^{-1})W^\sperp  \phi(g) x) - \sigma(Wx)}_2^2 \dd \lambda(g)\\ 
        &\le  C^2 \int_\G  \norm{\overline{W}x + \psi(g^{-1})W^\sperp  \phi(g) x - Wx}_2^2 \dd \lambda(g)\\ 
        &=  C^2\int_\G  \norm{(\psi(g^{-1})W^\sperp  \phi(g) - W^\sperp) x}_2^2 \dd \lambda(g)\\ 
    \end{align*}
    Then by an application of Fubini's theorem and the covariance of $X$ we get
    \begin{align*}
        &\E[\norm{\q f_W(X) - f_W(X) }_2^2]  \\ 
        &\le  C^2\int_\G  
    \tr\left((\psi(g^{-1})W^\sperp  \phi(g) - W^\sperp)^\top \E[XX^\top] (\psi(g^{-1})W^\sperp  \phi(g) - W^\sperp) \right)\dd \lambda(g)\\
        & = C^2 \fnorm{W^\sperp}^2 + C^2 \int_\G  \fnorm{\psi(g^{-1})W^\sperp  \phi(g) }^2\dd \lambda (g) 
    \end{align*}
    and then the argument of the integral can be analysed as
    \begin{align*}
        \fnorm{\psi(g^{-1})W^\sperp  \phi(g)}^2
        &= \tr\left((\psi(g^{-1})W^\sperp  \phi(g))^\top \psi(g^{-1})W^\sperp  \phi(g)\right)\\
        &= \tr\left(\phi(g^{-1})(W^\sperp)^\top  \psi(g)\psi(g^{-1})W^\sperp  \phi(g)\right)\\
        &= \tr\left((W^\sperp)^\top  W^\sperp \right).
    \end{align*}
    The proof is complete.
\end{proof}

\Cref{prop:regularisation} shows that the distance between the outputs of
a single layer neural network and its closest equivariant function is bounded by the norm
of the anti-symmetric component of the weights $W^\sperp$.
This quantity can be interpreted as a measure of the equivariance of the layer
and regularising $\fnorm{W^\sperp}$
will encourage the network to become (approximately)
equivariant.
It is easy to generalise~\cref{prop:regularisation} so that $X$ follows any $\G$-invariant
distribution with finite second moment. 

\subsection{Open Questions}
\paragraph{Equivariant Convolutions}
There has been much work on engineering convolutional layers to be group
equivariant, for instance~\citet{cohen2016group,cohen2018spherical,kondor2018generalization,cohen2019general}.
The convolution is a linear operator parameterised by the kernel.
This suggests that it may be possible to analyse
the generalisation properties of group equivariant convolutions in the
framework of~\cref{thm:fa}, similar to~\cref{sec:lower-bounds}.

\paragraph{Invariant/Equivariant Networks}
We have discussed enforcing invariance/equivariance in a neural network
$F_{(W^1, \dots, W^L)}$ (with the dependence on the weights now explicit)
by restricting weight matrices to intertwine between representations at each layer.
We ask: is this the best way to encode symmetry?
Mathematically, let $X \sim \mu$ with $\G$-invariant $\mu$ and embed the functions realised by the network in
$V= S \oplus A$.  Given an invariant/equivariant target $s\in S$, must the 
best approximating neural network be layer-wise invariant/equivariant?
That is, are there $s \in S$ such that the following holds
\if\submission1
\begin{align*}
    &\inf_{\mathcal{W}}\E[\norm{F_{(W^1, \dots, W^L)}(X) - s(X)}^2_2] \\
    &< \inf_{\mathcal{U}} \E[\norm{F_{(U^1, \dots, U^L)}(X) - s(X)}_2^2], 
\end{align*}
\else
\[
    \inf_{\mathcal{W}}\E[\norm{F_{(W^1, \dots, W^L)}(X) - s(X)}^2_2] 
    < \inf_{\mathcal{U}} \E[\norm{F_{(U^1, \dots, U^L)}(X) - s(X)}_2^2], 
\]
\fi
where $\mathcal{W} = \{ W^l \in \R^{\kappa_l \times \kappa_{l+1}}:\, l =1, \dots, L\}$
is the set of all possible weight matrices and 
$\mathcal{U} = \{ U^l \in \twine^l(\R^{\kappa_l \times \kappa_{l+1}}):\, l =1, \dots, L\}$
is the set of all weight matrices restricted to be intertwiners?
A resolution to this might shed light on new ways of encoding symmetry in neural networks.

\section*{Acknowledgements}
We would like to thank Varun Kanade and Yee Whye Teh for advice and support throughout this project.
We also thank Will Sawin for providing%
\if\submission1%
a useful lemma (in the supplementary material),
\else%
~\cref{lemma:proj-variance},
\fi
Michael Hutchinson for instructive feedback on an earlier draft and 
Bobby He for a helpful discussion on inverse Wishart matrices.
BE receives support from the UK EPSRC CDT 
in Autonomous Intelligent Machines and Systems (grant reference EP/L015897/1).
SZ receives support from the Aker Scholarship.

\begin{supplementary}
\appendix
\section{Useful Facts}\label{sec:useful-facts}
\subsection{Inverses}
\begin{lemma}\label{lemma:pseudo-inverse}
    Let $D \in \R^{d \times d}$ be orthogonal and let $B \in \R^{d \times d}$ be any symmetric matrix,
    then
    \[
        (DBD^\top)^+ = DB^+ D^\top.
    \]
\end{lemma}
\begin{proof}
    Set $X = DB^+ D^\top$ and $A = DBD^\top$.
    It suffices to check that $A$ and $X$ satisfy the Penrose equations,
    the solution of which is unique~\cite{penrose1955generalized}, namely:
    \begin{enumerate*}
        \item $AXA$ = $A$,
        \item $XAX$ = $X$,
        \item $(AX)^\top = AX$ and
        \item $(XA)^\top = XA$.
    \end{enumerate*}
    It is straightforward to check that this is the case.
\end{proof}

\begin{lemma}[\cite{gupta1968some}]\label{lemma:expected-inv-wishart}
    Let $\x \in \R^{n \times d}$ have \iid~$\n(0, 1)$ elements with $n > d + 1$. Then
    \[
        \E[(\x^\top \x)^+] = \frac{1}{n - d - 1}I.
    \]
\end{lemma}
\begin{remark}
    It is well known that the expectation in~\Cref{lemma:expected-inv-wishart} diverges for $d \le n \le d+1$. 
    To see this, first notice that since the normal distribution is $O(d)$ invariant
    $R\E[(\x^\top \x)^+]R^\top = \E[(\x^\top \x)^+]$ for any $R \in O(d)$.
    Hence $\E[(\x^\top \x)^+]$ is a scalar multiple of the identity: it is symmetric
    so diagonalisable, hence diagonal in every basis by the invariance, then 
    permutation matrices can be used to show the diagonals are all equal.
    It remains to consider the eigenvalues. 
    The eigenvalues $\lambda_1, \dots, \lambda_d$ of $\x^\top \x$ have 
    joint density (w.r.t.~Lebesgue) that is proportional to 
    \[
        \exp\left(-\frac12 \sum_{i=1}^d \lambda_i\right) \prod_{i=1}^d \lambda_i^{(n - d - 1)/2} \prod_{i < j}^d \abs{\lambda_i - \lambda_j}  
    \]
    when $n \ge d$ and $0$ otherwise, e.g.~see~\citet[Corollary 3.2.19]{muirhead2009aspects}.
    We need to calculate the mean of $1 / \lambda$ with respect to this density, 
    which diverges unless $n \ge d+2$. Taking the mean of $\lambda_k$, there is a term from the Vandermonde product that
    does not contain $\lambda_k$, so the integrand in the expectation goes like $\sqrt{\lambda_k^{n - d -3}}$ as $\lambda_k \to 0$.
\end{remark}

\begin{lemma}[{\cite[Theorem 2.1]{cook2011mean}}]\label{lemma:expected-inv-wishart-singular}
    Let $\x \in \R^{n \times d}$ have \iid~$\n(0, 1)$ elements with $n < d - 1$. Then
    \[
        \E[(\x^\top \x)^+] = \frac{n}{d(d - n - 1)}I.
    \]
\end{lemma}
\begin{remark}
   The statement of~\Cref{lemma:expected-inv-wishart-singular} in~{\citet[Theorem 2.1]{cook2011mean}}
   gives the condition $n < d - 3$, but this is not necessary for the first moment.
   This is easily seen by examining the proof of~{\citet[Theorem 2.1]{cook2011mean}}.
   In addition, the proof uses a transformation followed by an application of~\Cref{lemma:expected-inv-wishart}
   with the roles of $n$ and $d$ switched. It follows that the expectation diverges
   when $d \ge n \ge d-1$.
\end{remark}

\subsection{Projections}
\begin{lemma}\label{lemma:proj-variance}
    Let $E \sim \Unif \Gr_n(\R^d)$ where $0 < n < d$ 
    and let $P_E$ be the orthogonal projection onto
    $E$, then in components
    \[
        \E[P_E \otimes P_E]_{abce} =
        \frac{n(d-n)}{d(d-1)(d+2)}(\delta_{ab}\delta_{ce} +
         \delta_{ac}\delta_{be} +  \delta_{ae}\delta_{bc})
        +\frac{n(n-1)}{d(d-1)} \delta_{ab}\delta_{ce} 
    \]
\end{lemma}
\begin{proof}
    We use the Einstein convention of implicitly summing over repeated indices.
    The distribution of $E$ is orthogonally invariant, so $\E[P_E \otimes P_E]$ is isotropic.
    Thus, $\E[P_E \otimes P_E]$ must have components
    \[
        \Gamma_{abce}\coloneqq 
        \E[P_E \otimes P_E]_{abce}
        = \alpha \delta_{ab}\delta_{ce} + \beta \delta_{ac}\delta_{be} + \gamma \delta_{ae}\delta_{bc}
    \]
    e.g.~see~\citet{hodge1961}.
    Contracting indices gives
    \begin{alignat*}{2}
        &n^2 &&= \E[\tr(P_E)^2] = \Gamma_{aabb} = d^2 \alpha + d\beta + d\gamma \\ 
        &n &&= \E[\tr(P_E^\top P_E)] = \Gamma_{abab} = d\alpha + d^2\beta + d\gamma \\
        &n &&= \E[\tr(P_E^2)] = \Gamma_{abba} = d\alpha + d\beta + d^2\gamma
    \end{alignat*}
    from which one finds
    \begin{align*}
        \beta &= \frac{n(d-n)}{d(d-1)(d+2)}\\
        \alpha &= \beta + \frac{n(n-1)}{d(d-1)} \\ 
        \gamma &= \beta.
    \end{align*}
\end{proof}

\subsection{Component representation of $\twine$}\label{sec:cpts}
Using the Einstein convention of implicitly summing over repeated indices, one can write
\begin{align*}
    \twine(W)_{ab}
    &= \int_\G \phi(g)_{ac} W_{ce}\psi(g^{-1})_{eb}\dd \lambda(g) \\ 
    &= \int_\G \phi(g)_{ac} W_{ce}\psi(g)_{be} \dd \lambda(g) \quad \text{representation is orthogonal}\\ 
    &= \left (\int_\G \phi(g)_{ac} \psi(g)_{be} \dd \lambda(g) \right)W_{ce} \quad \text{components are scalars}\\ 
    &= \twine_{abce} W_{ce}
\end{align*}
where in the last line we identify the components of the $4$-tensor $\twine$.

\section{Additional Proofs}\label{sec:proofs-appdx}
In this section we give proofs of~\Cref{thm:fa,thm:fa-rademacher}.
Throughout this section, as in the rest of the paper,
$\G$ will be a compact, second countable and Hausdorff topological group.
There exists a unique left and right invariant
Haar measure $\lambda$ on $\G$~\cite[Theorem 2.27]{kallenberg2006foundations},
which we may normalise to be a probability measure $\lambda(\G) = 1$.
The Haar measure is Radon which means it is finite on any compact set,
so it is clearly normalisable $\lambda(\G) < \infty$.
This also immediately implies that $\lambda$ is $\sigma$-finite, allowing us to
use Fubini's theorem.\footnote{%
    Weaker technical conditions are possible to achieve $\sigma$-finite $\lambda$.
    See Section 2.3 of~\citet{folland2016course}.
}

\subsection{Proof of~\Cref{thm:fa}}\label{sec:fa-proof}
Let $\X$ be an input space and $\mu$ be a  $\sigma$-finite, $\G$-invariant measure on $\X$.
We consider vector-valued functions $f : \X \to \R^k$
for some $k \in \N$. 
Let $\inner{\cdot}{\cdot} :\R^k \times \R^k \to \R$ be an inner product on $\R^k$
and let $\norm{\cdot}$ be the induced norm.
It is possible for the inner product to vary with $x \in \X$, making the norm a local metric on $\X$,
as long as the inner product evaluated at any point $x \in \X$, i.e.~$\iota_{a, b}(x) = \inner{a}{b}(x)$,
is $\G$-invariant as a function of $x$ for any $a, b \in \R^k$.
We will consider the Bochner space $V$ of all integrable, normalisable $f : \X \to \R^k$.
By integrable we mean that $\int_\X \abs{\inner{a}{f(x)}}\dd \mu(x) < \infty$ $\forall a \in \R^k$
(this allows us to use Jensen's inequality).
To construct the Bochner space we define an inner product 
\[
    \inner{f}{h}_\mu = \int_\X \inner{f(x)}{h(x)} \dd \mu(x)
\]
with corresponding norm $\norm{f}_\mu = \sqrt{\inner{f}{f}_\mu}$ and set $V$ to be the space
of all $f$ with $\norm{f}_\mu < \infty$.
Two feature maps are equal if they disagree only on sets with $\mu$-measure 0.

Let the measurable map $\psi: \G \to \text{GL}_k(\R)$ be a representation of $\G$.
We will assume that $\psi$ is unitary with respect to $\inner{\cdot}{\cdot}$, by
which we mean that $\forall a, b \in \R^k$ and $\forall g \in \G$
\[
    \inner{\psi(g) a}{\psi(g)b} = \inner{a}{b}. 
\]
Notice that this implies $\inner{\psi(g)a}{b} = \inner{a}{\psi(g^{-1})b}$.
If $\inner{\cdot}{\cdot}$ is the Euclidean inner product, then this is the usual
notion of a unitary representation.
The assumption of unitarity is not stringent, since one can always apply the Weyl trick.
We say that $f \in V$ is equivariant if $f(gx) = \psi(g)f(x)$ $\forall g\in \G,  x \in \X$.
Define the operator $\q: V\to V$ to have values
\[
    (\q f)(x) = \int_\G \psi(g^{-1}) f(g x) \dd \lambda(g) .
\]
For convenience we will write this as $\q f (x) = \E[\psi(G^{-1}) f(Gx)]$ where $G \sim \lambda$ will be distributed according 
to the Haar measure on $\G$.
The developments of this section apply to $\O: f(x) \mapsto \int_\G f(g x) \dd \lambda(g)$ by letting $\psi$ be the trivial representation. 

We first check that $\q$ is well-defined.
\begin{proposition}\label{prop:basic-conditions}
    Let $f \in V$, then
    \begin{enumerate}
        \item $\q f$ is $\mu$-measurable,
            and
        \item $\q f \in V$.
    \end{enumerate}
\end{proposition}
\begin{proof}
    \mbox{}
    \begin{enumerate}
        \item Writing the action $\phi$ of $\G$ on $\X$ explicitly, 
            the function $\psi \circ f \circ \phi: \G \times \X \to \Y$ with $\psi \circ f \circ \phi: (g, x) \mapsto \psi(g^{-1})f(\phi(g)x)$
            is $(\lambda \otimes \mu )$-measurable,
            so $\q f$ is $\mu$-measurable by~\citet[Lemma 1.26]{kallenberg2006foundations}.
        \item We apply in sequel Jensen's inequality~\cite[Lemma 3.5]{kallenberg2006foundations},
            the unitarity of $\psi$,
            Fubini's theorem~\cite[Theorem 1.27]{kallenberg2006foundations} and
            finally the invariance of $\mu$.
            \begin{align*}
                \norm{\q f}^2_\mu
                &= \int_\X \norm{\E[\psi(G^{-1}) f(Gx)]}^2 \dd\mu(x) \\ 
                &\le \int_\X \E[\norm{\psi(G^{-1}) f(Gx)}^2 ]\dd\mu(x) \\ 
                &= \int_\X \E[\norm{f(Gx)}^2 ]\dd\mu(x) \\ 
                &= \int_\X \norm{f(x)}^2 \dd\mu(x) \\
                &= \norm{f}_\mu^2 < \infty
            \end{align*}
    \end{enumerate}
\end{proof}

\begin{proposition}\label{prop:sym-cpt}
    $f$ is equivariant if and only if $\q f = f$.
\end{proposition}

\begin{proof}
    Suppose $f$ is equivariant then $f(gx) = \psi(g)f(x)$ $\forall g\in \G$, $\forall x \in \X$.
    Hence for any $x \in \X$
    \[
        \q f(x) = \E[\psi(G^{-1})f(G x)] 
        = \E[\psi(G^{-1})\psi(G)f(x)] 
        = f(x).
    \]
    Now assume that $\q f = f$, so for any $x \in \X$
    $
    f(x) = \E[\psi(G^{-1})f(Gx)]
        $.
    Take any $h \in \G$, then 
    \begin{align*}
        f(hx) &= \E[\psi(G^{-1})f(Ghx)] \\
            &= \psi(h)\E[\psi((Gh)^{-1})f(Ghx)] \\
            &= \psi(h)\E[\psi(G^{-1})f(Gx)] \\
            &= \psi(h)f(x)
    \end{align*}
   where in the third line we used the right invariance of the Haar measure. 
\end{proof}

\begin{proposition}\label{prop:evals}
    $\q$ has only two eigenvalues, $1$ and $0$.
\end{proposition}

\begin{proof}
    By~\Cref{prop:sym-cpt}, $\q^2 f = \q f$. So $\q f = \lambda f$ implies $\lambda^2 = \lambda$.
\end{proof}

Let $S$ and $A$ be the eigenspaces with eigenvalues $1$ and $0$ respectively.
Any $f \in V$ can be written $f = \bar{f} + f^\sperp$ where $\bar{f} = \q f$ and
$f^\sperp = f - \q f$. This implies that $V = S + A$.
We conclude by showing that $\q$ is self-adjoint with respect to $\inner{\cdot}{\cdot}$.
\Cref{thm:fa} follows immediately, since if $f \in S$ and $h \in A$ then
\[
    \inner{f}{h}_\mu = \inner{\q f}{h}_\mu = \inner{f}{\q h}_\mu = \inner{f}{0}_\mu = 0.
\]

\begin{proposition}\label{prop:self-adjoint}
    $\q$ is self-adjoint with respect to $\inner{\cdot}{\cdot}_\mu$.
\end{proposition}

\begin{proof}
    \begin{align*}
        \inner{\q f}{h}_\mu
        &= \int_{\X} \inner{\E[\psi(G^{-1}) f(Gx)]}{h(x)} \dd \mu(x) \\
        &= \int_{\X} \E[\inner{f(Gx)}{\psi(G) h(x)}] \dd \mu(x) \\
        &= \int_{\X} \inner{f(x)}{\E[\psi(G) h(G^{-1}x)]} \dd \mu(x) \\
    \end{align*}
    where we have used the unitarity of $\psi$ and the invariance of $\mu$.
    We conclude with the following claim.
    \begin{claim*}
        For any $x \in \X$
        \[
            \E[\psi(G) h(G^{-1}x)] = (\q h) (x).
        \]
    \end{claim*}
    \begin{proof}[Proof of claim]
        \renewcommand{\qedsymbol}{$\blacksquare$} 
        Since $\G$ is compact it is unimodular~\cite[Corollary 2.28]{folland2016course}.
        Let $A \subset \G$ be measurable, then by~\citet[Proposition 2.31]{folland2016course}
        $
            \lambda(\{a^{-1}: a \in A\}) = \lambda (A).
            $
        So we can just make the change of variables $g \mapsto g^{-1}$
        \begin{align*}
            \E[\psi(G) h(G^{-1}x)] 
                &= \int_\G \psi(g) h(g^{-1}x) \dd \lambda(g) \\
                &= \int_\G \psi(g^{-1}) h(gx) \dd \lambda(g^{-1}) \\
                &= \int_\G \psi(g^{-1}) h(gx) \dd \lambda(g) \\
                &= (\q h) (x).
        \end{align*}
    \end{proof}
\end{proof}

\subsubsection{A Note on the Proof of~\cref{thm:fa}: $\G$-invariance of $\mu$ is necessary}
We discuss the importance of this condition on the underlying measure.

The main usage of the $\G$-invariance of $\mu$ in the proof is in~\cref{prop:self-adjoint},
where we show that $\q$ is self-adjoint with respect to $\inner{\cdot}{\cdot}_\mu$.
We give an example showing that $\G$-invariance is needed for the orthogonal
decomposition $V = S \oplus A$, but first show how it is needed for the main 
results.

Let $S = \q V$.
In the main results we make heavy use of the form of the projection onto
$S^\sperp$ being $I - \q$.
This is equivalent to $\q a = 0$ $\forall a \in S^\sperp$
which in turn is equivalent to $\q$ being self-adjoint.
\begin{proof}
    Let $a \in S^\perp$. If $a = (I - \q) f$ for some $f \in V$ then clearly
    $\q a = 0$. On the other hand, any $f \in V$ has 
    $f = s + a$ for some $s \in S$ and $a \in S^\sperp$,
    but then $\q a =0$ implies $s = \q f$ and $a = (I - \q) f$. 
    This gives the first equivalence.

    Now if $\q$ self-adjoint then 
    $\norm{\q a}_\mu^2= \inner{\q a}{\q a}_\mu = \inner{a}{\q a}_\mu = 0$
    because we must have $\q a \in S$.
    On the other hand, let $f_1 = s_1 + a_1$ and $f_2 = s_2 + a_2$
    where $s_i \in S$ and $a_i \in S^\sperp$, orthogonality gives
    \[
        \inner{\q f_1}{f_2}_\mu
        = \inner{s_1}{s_2 + a_2}_\mu
        = \inner{s_1}{s_2}_\mu
        = \inner{s_1 + a_1}{s_2}_\mu
        = \inner{f_1}{\q f_2}_\mu.
    \]
\end{proof}

In the following example, the spaces $S$ and $A$
are orthogonal with respect to $\inner{\cdot}{\cdot}_\mu$ if and only if
$\mu$ is $\G$-invariant.  
\begin{example}
    Let $C_2$ act on $\X = \{-1, 1\} \times \{1\}$ by
    multiplication on the first coordinate. Let $V$ be the vector space of
    `linear' functions $f_{(t_1, t_2)} : \X \to \R, (x_1, x_2)^\top \mapsto
    (t_1, t_2) (x_1, x_2)^\top$ for any $t_1, t_2 \in \R$. We note that any
    distribution $\mu$ on $\X$ can be described by its probability mass
    function $p$, which is defined by $p(-1, 1)$ and $p(1, 1)$. Moreover, $\mu$
    is $C_2$-invariant if and only if $p(-1, 1) = p(1, 1)$. Next, observe that
    the symmetric functions $S$ and anti-symmetric functions $A$ are precisely
    those for which $t_1 = 0$ and $t_2 = 0$ respectively. The inner product
    induced by $\mu$ is given by $\inner{f_{(a_1, a_2)}}{f_{(b_1, b_2)}}_\mu =
    (a_1 + a_2)(b_1 + b_2)p(1, 1) + (a_2 - a_1)(b_2 - b_1)p(-1, 1)$. With this,
    we see that the inner product $\inner{f_{(0, t_2)}}{f_{(t_1, 0)}}_\mu = t_1
    t_2 p(1, 1) - t_1 t_2 p(-1, 1)$ is zero for all $f_{(0, t_2)} \in S$ and
    $f_{(t_1, 0)} \in A$ if and only if $p(-1, 1) = p(1, 1)$, that is, if and
    only if $\mu$ is $\G$-invariant.  
\end{example}

\subsection{Proof of~\cref{thm:least-squares}}\label{sec:proof-least-squares}
\begin{proposition*}[\Cref{thm:least-squares}]
    \thmleastsquares{}
\end{proposition*}
\begin{proof}
    Using~\Cref{thm:fa} and~\Cref{prop:sym-cpt} we can write
    $
            S = \{ f \in V : \O f = f \}
    $
    which in turn implies that any $f \in V$ has the decomposition
    $
        f = s + a
        $
    where $\O s = s$, $\O a = 0$ and $\inner{s}{a}_\mu = 0$. Hence $\bar{f} = s = \O f$.
    Now take any $h \in S$ and recall that $\inner{h}{a}_\mu = 0$, then
        \[
        \norm{f - h}_\mu^2 
        = \norm{(s - h) + a}_\mu^2
        = \norm{s - h}_\mu^2 + \norm{a}_\mu^2
        \ge \norm{a}_\mu^2
        = \norm{f - s}_\mu^2.
        \]
    Uniqueness follows by a simple convexity argument.
    Suppose $\exists s' \in S$ with $s' \ne s$ and $\norm{f -s '}_\mu = \norm{f - s}_\mu$, then 
    since $S$ is a vector space we have $s_{\frac12} = \frac12(s + s') \in S$.
    It follows that
    \begin{align*}
        \norm{f - s_{\frac12}}_\mu^2 
        &= \norm{ (f - s)/2  + (f - s')/2}_\mu^2\\
        &= \frac14 \norm{f - s}_\mu^2 + \frac14 \norm{f - s'}_\mu^2 + \frac14 \inner{f - s}{f - s'}_\mu \\
        &\le \frac14 \norm{f - s}_\mu^2 + \frac14 \norm{f - s}_\mu^2 + \frac14 \norm{f -s}_\mu^2 \\
        & = \frac34 \norm{f - s}_\mu^2
    \end{align*}
    a contradiction unless $\norm{f - s}_\mu^2 = 0$, in which case
    $f = s$ $\mu$-almost-everywhere.
\end{proof}

\subsection{Proof of~\Cref{thm:fa-rademacher}}\label{sec:proof-fa-rademacher}
Let $\T$ be a space on which $\G$ acts measurably.
Let $f: \T \to \R$ be a integrable function.
Recall the orbit averaging of $f$ is the function $\O f : \T \to \R$ with
values
\[
    (\O f)(t) = \E[f(Gt)]
\]
where $G\sim\lambda$.
For any set $\F$ of integrable functions $f: \T \to \R$ we define the symmetric and
anti-symmetric classes as
\[
    \overline{\F} = \{\O f : f \in \F\} \quad \text{ and } \quad \F^\sperp = \{f - \O f: f \in \F\},
\]
respectively.
Notice that: \begin{enumerate*}
    \item by~\Cref{prop:sym-cpt}, $f$ is $\G$-invariant iff $\O f = f$,
    \item that everything in the symmetric class is preserved by $\O$ and everything
        in the anti-symmetric class vanishes under $\O$, and
    \item that $\O f$ is measurable whenever $f$ is measurable by~\Cref{prop:basic-conditions}.
\end{enumerate*}

\begin{proposition*}[\Cref{thm:fa-rademacher}]
    Suppose $\nu$ is a $\G$-invariant probability distribution on $\T$.
    Let $\T$ be some input space and let $\F$ be a set of $\nu$-integrable functions $f: \T \to \R$.
    Then the Rademacher complexity of the feature averaged class satisfies
   \[
        0 \le  \ERad_n (\F)  -  \ERad_n (\overline{\F})  \le  \ERad_n (\F^\sperp)
   \]
   where the expectations in the definition of $\ERad_n$ are taken over $t_i \sim \nu$ \iid.
\end{proposition*}

We start with the left hand side.
\begin{claim*}\label{claim:rad-sym}
    \[
        \ERad_n \overline{\F} \le \ERad_n \F .
    \]
\end{claim*}
\begin{proof}[Proof of claim]
    Let $t \in \T^n$. The action of $\G$ on $\T$ induces
    an action on $\T^n$ by $gt = (g t_1, \dots, g t_n)$.
    If $T \sim \nu^n$ then by $\G$-invariance of $\nu$ we have 
    $g T \eqdist T$ $\forall g \in \G$.
    Let $T_1, \dots, T_n \sim \nu$ with $T = \{T_1, \dots, T_n\}$.
    Then, with subscripts on expectations for clarity,
    \begin{align*}
        \ERad_n \overline{\F} 
        &= \frac1n \E_{T} \Esig \sup_{\bar{f}\in \overline{\F}} 
        \abs{\sum_{i=1}^n \sigma_i \bar{f}(T_i)}\\
        &= \frac1n \E_{T} \Esig \sup_{f\in \F} \abs{\E_G\sum_{i=1}^n \sigma_i {f(G T_i)}}\\
        &\le \frac1n \E_{T} \Esig  \E_G \sup_{f\in \F}\abs{\sum_{i=1}^n \sigma_i {f(G T_i)}}\\
        &= \frac1n \E_{T} \Esig  \sup_{f\in \F} \abs{\sum_{i=1}^n \sigma_i {f(T_i)}}
        \label{line:fubini-invariance} \tag{$\star$} \\
        &= \ERad_n \F 
    \end{align*}
    In deducing~(\ref{line:fubini-invariance}) we used 
    Fubini's theorem~\cite[Theorem 1.27]{kallenberg2006foundations} and
    the $\G$-invariance of $\nu$.
    Fubini's theorem applies by~\cite[Lemma 1.26]{kallenberg2006foundations}.
\end{proof}

Now for the right hand side.
\begin{proof}
    For any $f \in \F$ we can write $f = \bar{f} + f^\sperp$ where $\bar{f} \in \overline{\F}$ 
    and $f^\sperp \in \F^\sperp$. Then, for any $\tau = \{t_1, \dots, t_n\}$
    \begin{align*}
        \Rad_\tau \F  
        &= \frac1n \Esig \sup_{f\in \F} \abs{\sum_{i=1}^n \sigma_i f(t_i)} \\
        &= \frac1n \Esig \sup_{f\in \F} \abs{\sum_{i=1}^n \sigma_i (\bar{f}(t_i) + f^\sperp(t_i)) }\\
        &\le \frac1n \Esig \sup_{\bar{f}\in \overline{\F}} \abs{\sum_{i=1}^n \sigma_i \bar{f}(t_i)}
        + \frac1n \Esig{} \sup_{f^\sperp \in \F^\sperp}\abs{ \sum_{i=1}^n \sigma_i  f^\sperp(t_i) }\\
        &= \Rad_\tau \overline{\F}  + \Rad_\tau \F^\sperp 
    \end{align*}
    Taking an expectation over $\tau \sim \nu^n$ and combining with the claim gives
    \[
        \ERad_n \overline{\F} \le 
        \ERad_n \F  
        \le \ERad_n \overline{\F}  + \ERad_n \F^\sperp 
    \]
    from which the proposition follows immediately.
\end{proof}

\subsection{Proof of~\Cref{thm:equivariant-regression}}\label{sec:equivariant-proof}
\begin{theorem*}[\Cref{thm:equivariant-regression}]
\thmequivariantregression
\end{theorem*}

\begin{proof}
    \newcommand{\xxi}{\bm{\xi}}
    \newcommand{\z}{\bm{Z}}
    We use Einstein notation, in which repeated indices are summed over.
    $\delta_{ij}$ represents the Kronecker delta, which is $1$ when $i = j$ and $0$ otherwise.

    Since the representation $\phi$ is orthogonal, $X$ is $\G$-invariant for any $\G$.
    We have seen from~\Cref{prop:generalisation-gap} that 
    \[
        \E[\Delta(f_W, f_{\overline{W}}) ]= \sigma_X^2\E[ \fnorm{W^\sperp}^2]
    \]
    and we want to understand this quantity for the least-squares estimate
    \[
        W = (\x^\top \x)^+ \x^\top \y =  (\x^\top \x)^+ \x^\top \x \Theta +  (\x^\top \x)^+ \x^\top \xxi
    \]
    where $\x \in\R^{n \times d}$, $\y \in \R^{n \times k}$ are the row-stacked training examples
    with $(\x)_{ij} = (X_i)_j$, $(\y_i)_j = (Y_i)_j$ and $\xxi = \y - \x \Theta$.
    We have
    \begin{align*}
        \E[\Delta(f_W, f_{\overline{W}}) ]
        &= \sigma_X^2\E[ \fnorm{\twine^\sperp(W)}^2]\\
        &= \sigma_X^2\E[ \fnorm{\twine^\sperp((\x^\top \x)^+ \x^\top \x \Theta +  (\x^\top \x)^+ \x^\top \xxi)}^2]\\
        &= \sigma_X^2\E[ \fnorm{\twine^\sperp((\x^\top \x)^+ \x^\top \x \Theta)}^2] 
        + \sigma_X^2\E[\fnorm{\twine^\sperp ((\x^\top \x)^+ \x^\top \xxi)}^2]
    \end{align*}
    using linearity and $\E[\xxi] = 0$. We treat the two terms separately, starting with the second.

    \paragraph{Second Term}
    Now consider the second term, setting $\z = (\x^\top \x)^+ \x^\top$ we have 
    \[
        \E[\fnorm{\twine^\sperp (\z \xxi)}^2]
        = \E[\tr(\twine^\sperp (\z \xxi)^\top \twine^\sperp (\z\xxi))].
    \]
    One gets
    \begin{align}
        \E[\tr(\twine^\sperp (\z \xxi)^\top \twine^\sperp (\z\xxi))]
        &= \E[\twine^\sperp_{abcj} \z_{ce}\xxi_{ej} \twine^\sperp_{abfg} \z_{fh}\xxi_{hg}] \nonumber \\ 
        &= \sigma_\xi^2 \E[\twine^\sperp_{abcj} \z_{ce} \twine^\sperp_{abfg} \z_{fh} \delta_{eh}\delta_{jg}] 
            \quad \text{integrating $\xxi$}\nonumber\\ 
        &= \sigma_\xi^2 \twine^\sperp_{abcj} \twine^\sperp_{abfj} \E[\z_{ce}  \z_{fe} ] \nonumber\\
        &= \sigma_\xi^2 \twine^\sperp_{abcj} \twine^\sperp_{abfj} \E[(\z \z^\top)_{cf}] \label{eq:dagger} 
    \end{align}
    and then (WLOG relabelling $f \mapsto e$)
    \begin{align*}
        \twine^\sperp_{abcj} \twine^\sperp_{abej}
        &= \left(\delta_{ac}\delta_{bj} -  \int_\G  \phi(g)_{ac}\psi(g)_{bj}\dd \lambda(g) \right) 
           \left(\delta_{ae}\delta_{bj} -  \int_\G  \phi(g)_{ae}\psi(g)_{bj}\dd \lambda(g) \right)  \\ 
        &= \delta_{ac}\delta_{bj}\delta_{ae}\delta_{bj}                                     
         -  \delta_{ae}\delta_{bj}\int_\G  \phi(g)_{ac}\psi(g)_{bj}\dd \lambda(g) \\
        &\phantom{=} - \delta_{ac}\delta_{bj}\int_\G  \phi(g)_{ae}\psi(g)_{bj} \dd \lambda(g) 
        + \int_\G \phi(g_1)_{ac} \phi(g_2)_{ae} \psi(g_1)_{bj}\psi(g_2)_{bj}\dd \lambda(g_1) \dd \lambda(g_2) \\
        &= k\, \delta_{ce} - \int_\G \tr(\psi(g)) ( \phi(g)_{ec} + \phi(g)_{ce})\dd \lambda (g) \\
        &\phantom{=}+ \int_\G \tr(\psi(g_1)^\top \psi(g_2)) (\phi(g_1)^\top \phi(g_2))_{ce} \dd \lambda(g_1) \dd \lambda(g_2) 
    \end{align*}
    where we have used that the indices $b, j = 1,\dots, k$. Consider the final term
    \begin{align*}
         \int_\G \tr(\psi(g_1)^\top \psi(g_2)) (\phi(g_1)^\top \phi(g_2))_{ce} \dd \lambda(g_1) \dd \lambda(g_2) 
         &=  \int_\G \tr(\psi(g_1^{-1}g_2))(\phi(g_1^{-1}g_2))_{ce} \dd \lambda(g_1) \dd \lambda(g_2)  \\ 
         &= \int_\G \tr(\psi(g)) \phi(g)_{ce}\dd \lambda(g)  
    \end{align*}
    where we used that the representations are orthogonal, Fubini's theorem and that the Haar measure is invariant 
    ($\G$ is compact so $\lambda$ is $\sigma$-finite and Fubini's theorem applies).
    Now we put things back together. To begin with
    \[
        \twine^\sperp_{abcj} \twine^\sperp_{abej} = k\, \delta_{ce} - \int_\G \tr(\psi(g)) \phi(g^{-1})_{ce}\dd \lambda (g) 
    \]
    and putting this into~\cref{eq:dagger} with $\z \z^\top = (\x^\top \x)^{+}$ gives
    \[
        \E[\tr(\twine^\sperp (\z \xxi)^\top \twine^\sperp (\z\xxi))]
        = \sigma_\xi^2 \left( k\, \delta_{ce} - \int_\G  \tr(\psi(g)) \phi(g^{-1})_{ce} \dd \lambda (g)\right) \E[(\x^\top\x)^+_{ce}] 
    \]
    where $c, e = 1, \dots, d$. Applying~\Cref{lemma:expected-inv-wishart,lemma:expected-inv-wishart-singular}
    gives $\E[(\x^\top\x)^+_{ce}] = \sigma_X^{-2} r(n,d)\delta_{ce}$ where 
    \[
        r(n, d) = \begin{cases}
            {\frac{n}{d(d - n - 1)}} & n < d- 1\\
            (n - d - 1)^{-1} & n > d + 1 \\ 
            \infty & \text{otherwise}
        \end{cases}.
    \]
    When $n \in [d-1, d+1]$ it is well known that the expectation diverges, 
    see~\Cref{sec:useful-facts}.
    We can conclude:
    \[
        \sigma_X^2\E[\fnorm{\twine^\sperp ((\x^\top \x)^+ \x^\top \xxi)}^2] = 
        \sigma_\xi^2 r(n, d) \left(dk - \int_\G  \tr(\psi(g)) \tr(\phi(g))\dd \lambda (g)\right)
    \]
    where we have used the orthogonality of $\phi$.

    \paragraph{First Term}
    If $n \ge d$ then $(\x^\top \x)^+ \x^\top \x \Theta \eqas \Theta$ and since $h_\Theta \in S$ 
    the first term vanishes almost surely. This gives the case of equality in the statement.
    If $n < d$ we proceed as follows.
    Write $P_E = (\x^\top \x)^+ \x^\top \x$ which is the orthogonal projection onto
    the rank of $\x^\top \x$.
    By isotropy of $X$, $E \sim \Unif \Gr_n(\R^d)$ with probability 1.
    Recall that $\twine^\sperp(\Theta) = 0$, which in components reads
    \begin{equation}\label{eq:null}
        \twine^\sperp_{abce}\Theta_{ce} = 0 \quad \forall a, b.
    \end{equation}
    Also in components, we have
    \[
        \E[ \fnorm{\twine^\sperp((\x^\top \x)^+ \x^\top \x \Theta)}^2] 
        = \twine^\sperp_{fhai}\twine^\sperp_{fhcj} \E[P_E \otimes P_E]_{abce} \Theta_{bi} \Theta_{ej}
    \]
    and using~\cref{lemma:proj-variance} we get
    \begin{align*}
        \E[ \fnorm{\twine^\sperp((\x^\top \x)^+ \x^\top \x \Theta)}^2] 
        &=
        \frac{n(d-n)}{d(d-1)(d+2)}
        \left(
         \twine^\sperp_{fhai}\twine^\sperp_{fhaj} \Theta_{bi} \Theta_{bj} + 
         \twine^\sperp_{fhai}\twine^\sperp_{fhbj} \Theta_{bi} \Theta_{aj}
        \right) \\ 
        &\phantom{=} + \frac{n(d-n) + n(n-1)(d+2)}{d(d-1)(d+2)}
         \twine^\sperp_{fhai}\twine^\sperp_{fhcj} \Theta_{ai} \Theta_{cj}.
    \end{align*}
    The final term vanishes using~\cref{eq:null}, while the first term is
    \[
         \twine^\sperp_{fhai}\twine^\sperp_{fhaj} \Theta_{bi} \Theta_{bj} 
          = (\Theta^\top \Theta)_{ij}\twine^\sperp_{fhai}\twine^\sperp_{fhaj}
    \]
    where
    \begin{align*}
        \twine^\sperp_{fhai}\twine^\sperp_{fhaj}
        &= \left(\delta_{fa}\delta_{hi} -  \int_\G  \phi(g)_{fa}\psi(g)_{hi}\dd \lambda(g) \right) 
           \left(\delta_{fa}\delta_{hj} -  \int_\G  \phi(g)_{fa}\psi(g)_{hj}\dd \lambda(g) \right)  \\ 
        &= d \delta_{ij} - \int_\G \tr (\phi(g)) \psi(g)_{ij} \dd \lambda(g)
            - \int_\G \tr (\phi(g)) \psi(g)_{ji} \dd \lambda(g)\\
        &\phantom{=}+ \int_\G \phi(g_1)_{fa}\phi(g_2)_{fa} \psi(g_1)_{hi} \psi(g_2)_{hj} \dd \lambda(g_1)\dd \lambda(g_2) \\ 
        &= d \delta_{ij} - \int_\G \tr (\phi(g)) \psi(g)_{ji} \dd \lambda(g)
    \end{align*}
    using the orthogonality of the representations and invariance of the Haar measure.
    Therefore
    \begin{align*}
         \twine^\sperp_{fhai}\twine^\sperp_{fhaj} \Theta_{bi} \Theta_{bj} 
        &= d \fnorm{\Theta}^2 - \int_\G \chi_\phi(g)\tr\left(\psi(g^{-1}) \Theta^\top \Theta\right) \dd \lambda(g) \\ 
        &= d \fnorm{\Theta}^2 - \int_\G \chi_\phi(g)\tr\left(\psi(g) \Theta^\top  \Theta\right) \dd \lambda(g) .
    \end{align*}
    Now for the second part
    \begin{align*}
         \Theta_{bi} \Theta_{aj}\twine^\sperp_{fhai}\twine^\sperp_{fhbj}
        &= \Theta_{bi} \Theta_{aj} \left(\delta_{fa}\delta_{hi} -  \int_\G  \phi(g)_{fa}\psi(g)_{hi}\dd \lambda(g) \right) 
           \left(\delta_{fb}\delta_{hj} -  \int_\G  \phi(g)_{fb}\psi(g)_{hj}\dd \lambda(g) \right)  \\ 
        &= \Theta_{bi} \Theta_{aj}\left(\delta_{ab} \delta_{ij} - \int_\G \phi(g)_{ab}\psi(g)_{ij} \dd \lambda(g) \right) \\
        &= \fnorm{\Theta}^2 - \int_\G \tr( \Theta^\top \phi(g) \Theta \psi(g) ) \dd \lambda(g) \\
        &= \fnorm{\Theta}^2 - \int_\G \tr(\psi(g^2) \Theta^\top \Theta ) \dd \lambda(g).
    \end{align*}
    Putting these together gives 
    \[
      \E[ \fnorm{\twine^\sperp((\x^\top \x)^+ \x^\top \x \Theta)}^2] 
      = \frac{n(d-n)}{d(d-1)(d+2)} \left( (d+1) \fnorm{\Theta}^2- 
          \tr( J_\G \Theta^\top  \Theta) \right) 
    \]
    where $J_\G \in \R^{k \times k}$ is the matrix-valued function of $\G$, $\psi$ and $\phi$
    \[
        J_\G = \int_\G (\chi_\phi(g) \psi(g) + \psi(g^2)) \dd \lambda(g).
    \]
\end{proof}
\end{supplementary}
\bibliography{references}

\begin{thebibliography}{39}
\providecommand{\natexlab}[1]{#1}
\providecommand{\url}[1]{\texttt{#1}}
\expandafter\ifx\csname urlstyle\endcsname\relax
  \providecommand{\doi}[1]{doi: #1}\else
  \providecommand{\doi}{doi: \begingroup \urlstyle{rm}\Url}\fi

\bibitem[Abu-Mostafa(1993)]{abu1993hints}
Abu-Mostafa, Y.~S.
\newblock Hints and the {VC} dimension.
\newblock \emph{Neural Computation}, 5\penalty0 (2):\penalty0 278--288, 1993.

\bibitem[Anselmi et~al.(2014)Anselmi, Leibo, Rosasco, Mutch, Tacchetti, and
  Poggio]{anselmi2014unsupervised}
Anselmi, F., Leibo, J.~Z., Rosasco, L., Mutch, J., Tacchetti, A., and Poggio,
  T.
\newblock Unsupervised learning of invariant representations in hierarchical
  architectures, 2014.

\bibitem[Bartlett et~al.(2019)Bartlett, Harvey, Liaw, and
  Mehrabian]{bartlett2019nearly}
Bartlett, P.~L., Harvey, N., Liaw, C., and Mehrabian, A.
\newblock Nearly-tight {VC}-dimension and pseudodimension bounds for piecewise
  linear neural networks.
\newblock \emph{J. Mach. Learn. Res.}, 20:\penalty0 63--1, 2019.

\bibitem[Bloem-Reddy \& Teh(2020)Bloem-Reddy and Teh]{bloem2020probabilistic}
Bloem-Reddy, B. and Teh, Y.~W.
\newblock Probabilistic symmetries and invariant neural networks.
\newblock \emph{Journal of Machine Learning Research}, 21\penalty0
  (90):\penalty0 1--61, 2020.
\newblock URL \url{http://jmlr.org/papers/v21/19-322.html}.

\bibitem[Cohen \& Welling(2016)Cohen and Welling]{cohen2016group}
Cohen, T. and Welling, M.
\newblock Group equivariant convolutional networks.
\newblock In \emph{International conference on machine learning}, pp.\
  2990--2999, 2016.

\bibitem[Cohen et~al.(2018)Cohen, Geiger, K{\"o}hler, and
  Welling]{cohen2018spherical}
Cohen, T.~S., Geiger, M., K{\"o}hler, J., and Welling, M.
\newblock Spherical cnns.
\newblock \emph{arXiv preprint arXiv:1801.10130}, 2018.

\bibitem[Cohen et~al.(2019)Cohen, Geiger, and Weiler]{cohen2019general}
Cohen, T.~S., Geiger, M., and Weiler, M.
\newblock A general theory of equivariant cnns on homogeneous spaces.
\newblock In \emph{Advances in Neural Information Processing Systems}, pp.\
  9145--9156, 2019.

\bibitem[Cook et~al.(2011)Cook, Forzani, et~al.]{cook2011mean}
Cook, R.~D., Forzani, L., et~al.
\newblock On the mean and variance of the generalized inverse of a singular
  wishart matrix.
\newblock \emph{Electronic Journal of Statistics}, 5:\penalty0 146--158, 2011.

\bibitem[Folland(2016)]{folland2016course}
Folland, G.~B.
\newblock \emph{A course in abstract harmonic analysis}, volume~29.
\newblock CRC press, 2016.

\bibitem[Gupta(1968)]{gupta1968some}
Gupta, S.~D.
\newblock Some aspects of discrimination function coefficients.
\newblock \emph{Sankhy{\=a}: The Indian Journal of Statistics, Series A}, pp.\
  387--400, 1968.

\bibitem[Haasdonk et~al.(2005)Haasdonk, Vossen, and
  Burkhardt]{haasdonk05invariancein}
Haasdonk, B., Vossen, A., and Burkhardt, H.
\newblock Invariance in kernel methods by haar integration kernels.
\newblock In \emph{SCIA 2005, Scandinavian Conference on Image Analysis}, pp.\
  841--851. Springer-Verlag, 2005.

\bibitem[Hastie et~al.(2019)Hastie, Montanari, Rosset, and
  Tibshirani]{hastie2019surprises}
Hastie, T., Montanari, A., Rosset, S., and Tibshirani, R.~J.
\newblock Surprises in high-dimensional ridgeless least squares interpolation.
\newblock \emph{arXiv preprint arXiv:1903.08560}, 2019.

\bibitem[He et~al.(2016)He, Zhang, Ren, and Sun]{He16}
He, K., Zhang, X., Ren, S., and Sun, J.
\newblock Deep residual learning for image recognition.
\newblock In \emph{2016 IEEE Conference on Computer Vision and Pattern
  Recognition (CVPR)}, pp.\  770--778, 2016.
\newblock \doi{10.1109/CVPR.2016.90}.

\bibitem[Hodge(1961)]{hodge1961}
Hodge, P.~G.
\newblock On isotropic cartesian tensors.
\newblock \emph{The American Mathematical Monthly}, 68\penalty0 (8):\penalty0
  793--795, 1961.
\newblock ISSN 00029890.

\bibitem[Kallenberg(2006)]{kallenberg2006foundations}
Kallenberg, O.
\newblock \emph{Foundations of modern probability}.
\newblock Springer Science \& Business Media, 2006.

\bibitem[Kondor \& Trivedi(2018)Kondor and Trivedi]{kondor2018generalization}
Kondor, R. and Trivedi, S.
\newblock On the generalization of equivariance and convolution in neural
  networks to the action of compact groups.
\newblock In \emph{International Conference on Machine Learning}, pp.\
  2747--2755, 2018.

\bibitem[Lyle et~al.(2019)Lyle, Kwiatkowksa, and Gal]{lyle2019analysis}
Lyle, C., Kwiatkowksa, M., and Gal, Y.
\newblock An analysis of the effect of invariance on generalization in neural
  networks.
\newblock In \emph{International conference on machine learning Workshop on
  Understanding and Improving Generalization in Deep Learning}, 2019.

\bibitem[Lyle et~al.(2020)Lyle, van~der Wilk, Kwiatkowska, Gal, and
  Bloem-Reddy]{lyle2020benefits}
Lyle, C., van~der Wilk, M., Kwiatkowska, M., Gal, Y., and Bloem-Reddy, B.
\newblock On the benefits of invariance in neural networks, 2020.

\bibitem[Maron et~al.(2019)Maron, Fetaya, Segol, and
  Lipman]{maron2019universality}
Maron, H., Fetaya, E., Segol, N., and Lipman, Y.
\newblock On the universality of invariant networks.
\newblock In \emph{International Conference on Machine Learning}, pp.\
  4363--4371, 2019.

\bibitem[Mroueh et~al.(2015)Mroueh, Voinea, and Poggio]{mroueh2015learning}
Mroueh, Y., Voinea, S., and Poggio, T.~A.
\newblock Learning with group invariant features: A kernel perspective.
\newblock In \emph{Advances in Neural Information Processing Systems}, pp.\
  1558--1566, 2015.

\bibitem[Muirhead(2009)]{muirhead2009aspects}
Muirhead, R.~J.
\newblock \emph{Aspects of multivariate statistical theory}, volume 197.
\newblock John Wiley \& Sons, 2009.

\bibitem[Penrose(1955)]{penrose1955generalized}
Penrose, R.
\newblock A generalized inverse for matrices.
\newblock In \emph{Mathematical proceedings of the Cambridge philosophical
  society}, volume~51, pp.\  406--413. Cambridge University Press, 1955.

\bibitem[Pfau et~al.(2020)Pfau, Spencer, Matthews, and Foulkes]{pfau2020ab}
Pfau, D., Spencer, J.~S., Matthews, A.~G., and Foulkes, W. M.~C.
\newblock Ab initio solution of the many-electron schr{\"o}dinger equation with
  deep neural networks.
\newblock \emph{Physical Review Research}, 2\penalty0 (3):\penalty0 033429,
  2020.

\bibitem[Ravanbakhsh et~al.(2017)Ravanbakhsh, Schneider, and
  Poczos]{ravanbakhsh2017equivariance}
Ravanbakhsh, S., Schneider, J., and Poczos, B.
\newblock Equivariance through parameter-sharing.
\newblock In \emph{International Conference on Machine Learning}, pp.\
  2892--2901. PMLR, 2017.

\bibitem[Sannai \& Imaizumi(2019)Sannai and Imaizumi]{sannai2019improved}
Sannai, A. and Imaizumi, M.
\newblock Improved generalization bound of group invariant / equivariant deep
  networks via quotient feature space, 2019.

\bibitem[Sch{\"o}lkopf et~al.(1996)Sch{\"o}lkopf, Burges, and
  Vapnik]{scholkopf96incorporatinginvariances}
Sch{\"o}lkopf, B., Burges, C., and Vapnik, V.
\newblock Incorporating invariances in support vector learning machines.
\newblock pp.\  47--52. Springer, 1996.

\bibitem[Schulz-Mirbach(1992)]{schulz1992existence}
Schulz-Mirbach, H.
\newblock On the existence of complete invariant feature spaces in pattern
  recognition.
\newblock In \emph{International Conference On Pattern Recognition}, pp.\
  178--178. Citeseer, 1992.

\bibitem[{Schulz-Mirbach}(1994)]{schulz1994constructing}
{Schulz-Mirbach}, H.
\newblock Constructing invariant features by averaging techniques.
\newblock In \emph{Proceedings of the 12th IAPR International Conference on
  Pattern Recognition, Vol. 3 - Conference C: Signal Processing (Cat.
  No.94CH3440-5)}, volume~2, pp.\  387--390 vol.2, 1994.

\bibitem[Serre(1977)]{serre1977linear}
Serre, J.-P.
\newblock \emph{Linear representations of finite groups}.
\newblock Graduate texts in mathematics ; 42. Springer-Verlag, New York, 1977.
\newblock ISBN 9780387901909.

\bibitem[Simonyan \& Zisserman(2015)Simonyan and Zisserman]{Simonyan15}
Simonyan, K. and Zisserman, A.
\newblock Very deep convolutional networks for large-scale image recognition.
\newblock In \emph{International Conference on Learning Representations}, 2015.

\bibitem[Sokolic et~al.(2017)Sokolic, Giryes, Sapiro, and
  Rodrigues]{sokolic2017generalization}
Sokolic, J., Giryes, R., Sapiro, G., and Rodrigues, M.
\newblock Generalization error of invariant classifiers.
\newblock In \emph{Artificial Intelligence and Statistics}, pp.\  1094--1103,
  2017.

\bibitem[Szegedy et~al.(2015)Szegedy, Liu, Jia, Sermanet, Reed, Anguelov,
  Erhan, Vanhoucke, and Rabinovich]{Szegedy15}
Szegedy, C., Liu, W., Jia, Y., Sermanet, P., Reed, S., Anguelov, D., Erhan, D.,
  Vanhoucke, V., and Rabinovich, A.
\newblock Going deeper with convolutions.
\newblock In \emph{2015 IEEE Conference on Computer Vision and Pattern
  Recognition (CVPR)}, pp.\  1--9, 2015.
\newblock \doi{10.1109/CVPR.2015.7298594}.

\bibitem[Wadsley(2012)]{wadsleyrep2012}
Wadsley, S.
\newblock Lecture notes on representation theory, October 2012.
\newblock URL
  \url{https://www.dpmms.cam.ac.uk/~sjw47/RepThLecturesMich2012.pdf}.

\bibitem[Wainwright(2019)]{wainwright2019high}
Wainwright, M.~J.
\newblock \emph{High-dimensional statistics: A non-asymptotic viewpoint},
  volume~48.
\newblock Cambridge University Press, 2019.

\bibitem[Winkels \& Cohen(2018)Winkels and Cohen]{winkels20183d}
Winkels, M. and Cohen, T.~S.
\newblock 3d g-cnns for pulmonary nodule detection.
\newblock \emph{arXiv preprint arXiv:1804.04656}, 2018.

\bibitem[Wood \& Shawe-Taylor(1996)Wood and
  Shawe-Taylor]{wood1996representation}
Wood, J. and Shawe-Taylor, J.
\newblock Representation theory and invariant neural networks.
\newblock \emph{Discrete applied mathematics}, 69\penalty0 (1-2):\penalty0
  33--60, 1996.

\bibitem[Xu \& Mannor(2012)Xu and Mannor]{xu2012robustness}
Xu, H. and Mannor, S.
\newblock Robustness and generalization.
\newblock \emph{Machine learning}, 86\penalty0 (3):\penalty0 391--423, 2012.

\bibitem[Yarotsky(2018)]{yarotsky2018universal}
Yarotsky, D.
\newblock Universal approximations of invariant maps by neural networks, 2018.

\bibitem[Zaheer et~al.(2017)Zaheer, Kottur, Ravanbakhsh, Poczos, Salakhutdinov,
  and Smola]{zaheer2017deep}
Zaheer, M., Kottur, S., Ravanbakhsh, S., Poczos, B., Salakhutdinov, R.~R., and
  Smola, A.~J.
\newblock Deep sets.
\newblock In \emph{Advances in neural information processing systems}, pp.\
  3391--3401, 2017.

\end{thebibliography}
\bibliographystyle{icml2021}
\end{document}